\documentclass{article}
\usepackage{aditi}

\usepackage{microtype}
\usepackage{xspace}
\usepackage{hyperref}
\usepackage{url}
\usepackage{booktabs}

\usepackage{style}
\usepackage{algorithm}
\usepackage{algpseudocode}
\usepackage[most]{tcolorbox}
\usepackage{colortbl}
\usepackage{tablefootnote}
\usepackage{threeparttable}
\usepackage{tabularx}
\usepackage{multirow}
\usepackage{makecell}
\usepackage{listings}
\usepackage{stfloats}

\definecolor{skyblue}{RGB}{204,229,255}
\tcbset{
lightbluebox/.style={
colback=skyblue!70,
colframe=skyblue!75!black,
boxrule=1pt,
arc=4pt,
auto outer arc
}
}

\definecolor{TsinghuaPurple}{HTML}{660874}
\definecolor{mycolor_gray}{HTML}{ECECEC}
\definecolor{mydarkgreen}{RGB}{0,100,0}

\setContactAccent{TsinghuaPurple}

\hypersetup{
colorlinks=true,
citecolor=TsinghuaPurple,
linkcolor=TsinghuaPurple,
urlcolor=TsinghuaPurple
}

\AtBeginDocument{
\hypersetup{
colorlinks=true,
citecolor=TsinghuaPurple,
linkcolor=TsinghuaPurple,
urlcolor=TsinghuaPurple
}
}

\renewenvironment{abstract}{%
\vskip.02in
\begin{tcolorbox}[
enhanced,
colback=TsinghuaPurple!5!white,
colframe=TsinghuaPurple!5!white,
boxrule=0pt,
arc=8pt,
outer arc=8pt,
left=6mm,
right=6mm,
top=3.5mm,
bottom=4mm,
boxsep=0pt
]
\centerline{\large\bfseries Abstract}\vspace{0.3ex}%
}{%
\par
\end{tcolorbox}
\vskip 1ex
}

\DeclareCaptionFormat{ForumCaption}{%
\textbf{\color{TsinghuaPurple}#1#2}#3%
}
\newtcolorbox{mycustombox}[1]{
enhanced,
colback=black!5!white,
colframe=black!75!white,
boxrule=0.4pt,
coltitle=white,
title=#1,
titlerule=0.4pt,
fontupper=\small,
fonttitle=\small,
before upper={\par\smallskipamount},
breakable,
left=3mm,
right=3mm,
top=2mm,
bottom=2mm,
boxsep=1mm,
}

\setTitleruleGap{0.45pt}

\title{OmniVerifier-C1: Decomposed Graded Verifier for Rule-Governed Generative World Modeling}

\setauthors{
Bowei Liu$^{1,2*}$ \authorsep
Xinchen Zhang$^{1*\dagger}$ \authorsep
Xuhuan Li$^{1*}$ \authorsep
Tianlin Zhang$^{2}$ \authorsep
Kaian Jiang$^{1}$ \\ 
Jingwei Liu$^{2,4}$ \authorsep
Congyang Zhao$^{2}$ \authorsep 
Yifei Qian$^{2}$ \authorsep
Junhong Liu$^{3}$ \authorsep
Zhiheng Li$^{1}$ \\
Yujiu Yang$^{1\ddagger}$ \authorsep
Yang Cai$^{2}$ \authorsep 
Ling Yang$^{4\ddagger}$
}

\setaffils{
$^{1}$Tsinghua University \authorsep
$^{2}$ Amap, Alibaba Group \authorsep
$^{3}$ Microcyto \authorsep
$^{4}$ Peking University
\\
$^{*}$: Equal Contribution. \ \
$^{\dagger}$: Project Leader. \ \
$^{\ddagger}$: Corresponding Authors.
\vspace{-2mm}
}

\setcode{https://github.com/Cominclip/OmniVerifier}
\begin{document}

\maketitle

\vspace{-10mm}
\begin{abstract}

Generative visual verifiers play a crucial role in advancing multimodal intelligence. To enhance the robustness of visual verifiers in complex multimodal scenarios involving compositional requirements, we move beyond conventional binary verifiers that produce only True/False discrete judgments and introduce a verifier that automatically decomposes complex prompts or environmental rules into fine-grained components and assigns pointwise scores to provide a continuous and more informative critique signal. Our experiments further demonstrate that these structured and interpretable outputs improve the verifier's generalization to out-of-domain multimodal critic tasks.
To train this verifier, we introduce a Decoupled RL training strategy that jointly optimizes sub-question-level binary judgment and global alignment scoring, using precise local feedback to develop fine-grained verification capabilities that provide a foundation for reliable continuous scoring at the global level.
Under a general reinforcement learning framework, we first empirically demonstrate that the proposed pointwise graded verifier significantly outperforms binary verifiers in traditional text-to-image scenarios. We then extend our investigation to rule-governed generative world modeling where generative models are leveraged for world simulation and world reasoning under explicit rules. From a theoretical standpoint, we further show that a pairwise verification paradigm yields more accurate critiques than the pointwise formulation when ground-truth images are available, with empirical results on diverse structured world-modeling tasks such as Maze and Sudoku further validating the effectiveness of our approach. Overall, our findings suggest a key insight: the transition from generative models to structured world-modeling agents critically hinges on the availability of accurate pairwise graded visual verifiers.
\end{abstract}

\begin{figure}[h!!]
    \vspace{-5mm}
    \centerline{\includegraphics[width=\textwidth]{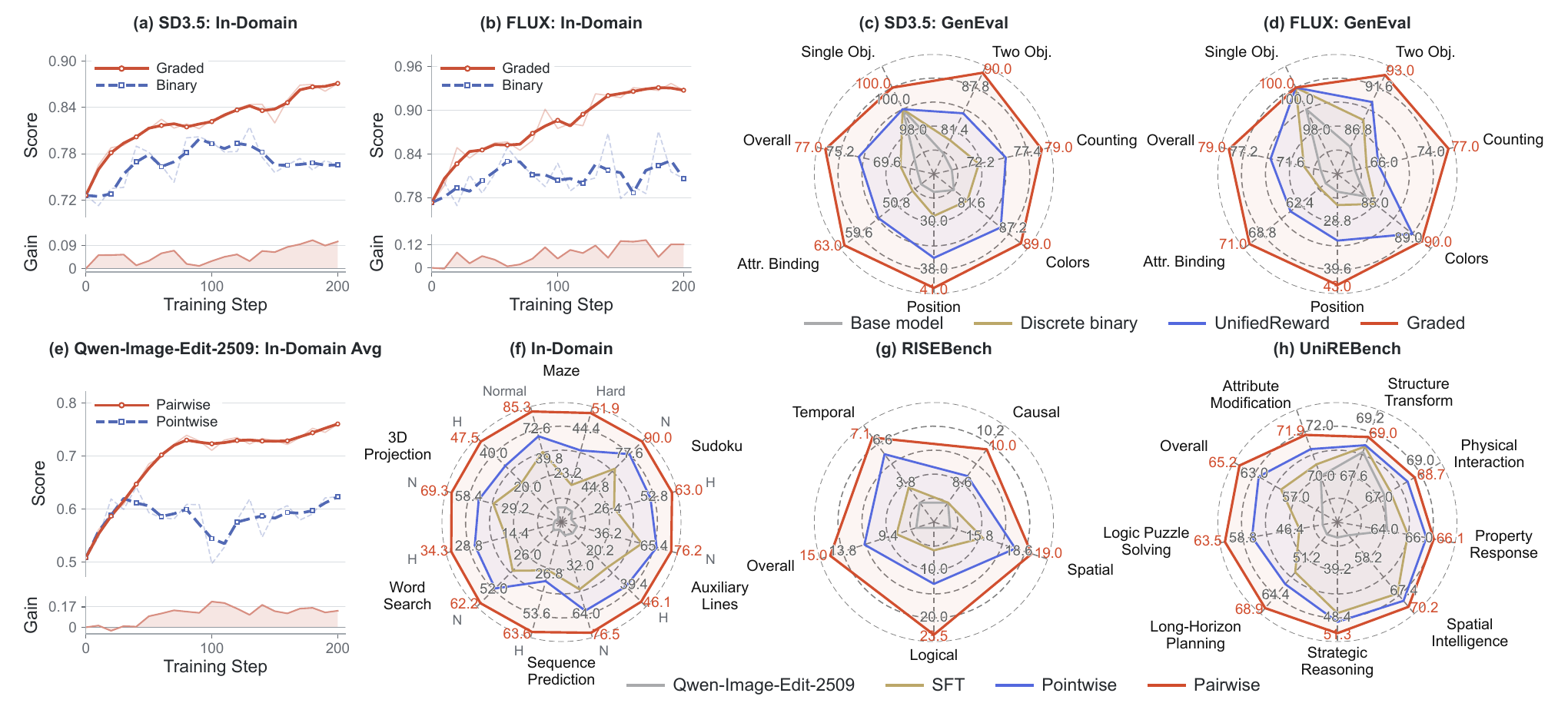}}
    \vspace{-4mm}
    \caption{Panels (a)--(d) highlight the advantages of the graded verifier over the binary verifier as a generative reward model, while panels (e)--(h) show that pairwise verification provides more reliable and accurate feedback signals than pointwise verification.}
    \vspace{-5mm}
    \label{fig1}
\end{figure}

\section{Introduction}
    \vspace{-2mm}
Multimodal generative models have made remarkable progress in recent years \citep{cao2025hunyuanimage,seedream2025seedream,chen2026semhitok}, producing visually compelling images conditioned on textual descriptions. Beyond visual fidelity \citep{wiedemer2025video, zhang2024realcompo, yang2025hermesflow}, a central challenge lies in whether generated images faithfully follow complex instructions and reflect coherent world knowledge \citep{zhou2025fireedit,jin2026unveilingfinegrainedvisualtraces, deng2025emerging, yang2025mmada,huang2026consistentid,zhang2026kvpo}. Motivated by this challenge, visual verifiers have emerged as a crucial component for assessing the correctness of generated results. By providing feedback on the alignment between images and textual descriptions, verifiers not only serve as evaluation tools, but also act as general critic models for inference-time scaling and reinforcement learning-based training. As generative models are increasingly applied to open-ended environments and reasoning-intensive tasks, the reliability and expressiveness of visual verifiers become a key bottleneck for advancing multimodal intelligence.

Despite their importance, most existing visual verifiers \citep{omniverifier, wang2025llava, xiong2025llava} adopt a binary judgment paradigm that produces only True/False decisions about whether an image satisfies a given prompt. Although simple and efficient, this design has several inherent limitations.
First, binary supervision provides coarse feedback that fails to capture the partial correctness of generated samples. By assigning the same negative label to imperfect samples of varying quality, binary rewards obscure meaningful quality differences and weaken the within-group learning signals required by group-relative reinforcement learning.
Moreover, the limited information conveyed by binary labels may encourage shortcut learning and increase susceptibility to reward hacking \citep{deepseekmathv2,MetaRM}, as models may exploit superficial patterns to predict the expected labels without adequately reasoning about the e subtl compositional relationships among objects, attributes, and spatial constraints.

To address these limitations, we propose the \textbf{Decomposed Graded Verifier (OmniVerifier-C1)}, which automatically decomposes complex prompts or environmental rules into fine-grained sub-questions and performs structured reasoning over each component. It assigns continuous scores based on the proportion of satisfied sub-questions, capturing partial correctness and providing informative critique signals. Our experiments demonstrate that decomposition-based structured reasoning significantly improves the verifier's out-of-domain multimodal critic capabilities. To train this verifier, we introduce a \textbf{Decoupled RL training strategy} that jointly optimizes sub-question-level binary judgment and global alignment scoring, using precise 0/1 feedback on individual sub-questions to develop strong fine-grained verification capabilities that support reliable global scoring. Moreover, our theoretical analysis shows that the decomposition-and-scoring paradigm better preserves the underlying true quality ordering among generated samples, providing finer-grained learning signals for group-relative reinforcement learning.

Building upon this verifier, we further explore its role as a generative reward model for reinforcement learning across two types of tasks. First, we apply the verifier to traditional text-to-image alignment, where it provides fine-grained feedback that encourages generative models to satisfy detailed compositional relationship. Second, we extend our study to the more reasoning-intensive and challenging setting of \textbf{structured generative world modeling}, where generative models are used to simulate environments and reason over sequential states. In these rule-governed environments, we find that pairwise verification, which compares generated states against ground-truth states, provides more accurate and reliable critiques than pointwise verification. This observation suggests that the effectiveness of visual verification in rule-governed world-modeling scenarios critically depends on the availability of reference states for comparison. Extensive experiments across diverse tasks, such as Maze and Sudoku, demonstrate that the proposed verifier provides more reliable reward signals and substantially improves reinforcement learning performance.

Our contributions can be summarized as follows:
\begin{itemize}

\item We propose the \textbf{Decomposed Graded Verifier}, which automatically decomposes complex prompts or environmental rules into fine-grained sub-questions, performs structured reasoning over each component, and assigns continuous scores based on the proportion of satisfied sub-questions. Our experiments demonstrate that its structured outputs significantly improve out-of-domain multimodal critic capabilities.

\item We introduce a \textbf{Decoupled RL training strategy} that jointly optimizes sub-question-level binary judgment and global alignment scoring, using precise 0/1 feedback on individual sub-questions to develop fine-grained verification capabilities that provide a foundation for reliable scoring.

    \item We demonstrate that the proposed verifier serves as an effective generative reward model for reinforcement learning, improving performance in both traditional text-to-image generation and rule-governed generative world modeling tasks.

    \item We provide empirical validation across diverse structured reasoning environments, showing that accurate \textbf{pairwise decomposed verifiers} are a key enabler for transforming generative models into rule-governed world-modeling agents.
    
\end{itemize}

\section{Related Work}
\vspace{-1mm}
\subsection{Generative Verifier or Reward Models}
\vspace{-1mm}
Generative reward models~\citep{Generativerewardmodels,Rewardbench2,tan2024judgebench,Rm-bench,MetaRM} differ fundamentally from traditional reward models \citep{xu2023imagereward,zhang2024generative,ouyang2022training, zhang2024itercomp} in that they produce explicit rationales rather than merely assigning scalar scores. OmniVerifier~\citep{omniverifier} first introduces the concept of Generative Universal Verifier and presents ViVerBench, a comprehensive and challenging benchmark for evaluating multimodal reasoning and critique capabilities. 
DeepSeekMath-V2~\citep{deepseekmathv2} further proposes meta-verification and integrates it into verifier training to enhance faithfulness and reliability.
OmniVerifier-M1 \citep{m1} extends meta-verification to multimodal scenarios by introducing symbolic spatial annotations, such as bounding boxes, points or lines as verifiable rationales. These rule-based supervision signals substantially alleviate reward hacking and enhance verifier generalization.
Despite their improved interpretability, existing generative verifiers still rely on coarse binary decisions as final judgments, and thus struggle to provide dense and optimization-friendly reward signals for reinforcement learning. 
In this paper, we introduce a Decoupled RL training strategy to train the Decomposed Graded Verifier, which provides reliable continuous scores through explicit fine-grained decomposition.

\subsection{Generative Model as World Reasoner}
\vspace{-1mm}
Multimodal generative models have made remarkable progress \citep{lin2025uniworld,wang2026liveedit,wang2026subjectanchor,zhang2026zo3t} in semantic alignment and aesthetic quality. Beyond visual fidelity, they are increasingly expected to serve as reasoning-capable world simulators \citep{sora2_system_card,nano_banana_pro_google,veo3_model_card}. 
To quantify the world reasoning capabilities of current generative models, several benchmarks \citep{t2i-corebench,t2i-reasonbench,han2026vigor} have been proposed. WISE \citep{wise} focuses on evaluating generative models' understanding of cultural commonsense and scientific knowledge.
MMGR \citep{mmgr} is designed to probe generative models' spatial, temporal, physical, and logical reasoning capabilities.
Motivated by these challenges, recent studies \citep{tan2026dreamworld,zhang2026robostereo,zhang2025mind} have explored new training paradigms  and input modalities to learn the underlying dynamics of world evolution.
Emu3.5 \citep{emu3.5} advances world modeling by learning native next-state prediction over long-horizon interleaved vision-language sequences.
UnityVideo \citep{huang2025unityvideo} improves world-aware video generation through unified multi-modal and multi-task learning across diverse visual representations.
In contrast to these efforts, we focus on improving the reliability and accuracy of reinforcement learning rewards for visual generation in rule-governed world-modeling scenarios, thereby encouraging generative models to develop stronger and more generalizable world reasoning capabilities.

\section{OmniVerifier-C1: Decomposed Graded Verifier} \label{sec 3}
\vspace{-1mm}
In this section, we first analyze the theoretical advantages of the graded verifier over the binary verifier as a generative reward model. We then empirically demonstrate that our verifier develops fine-grained and generalizable multimodal verification capabilities through the Decoupled RL strategy, which jointly optimizes sub-question-level binary judgment and global alignment scoring.

\subsection{Theoretical Advantage of Decomposed Graded Rewards for Generative RL} \label{sec 3.1}
\begin{figure}[t]
    \centerline{\includegraphics[width=\textwidth]{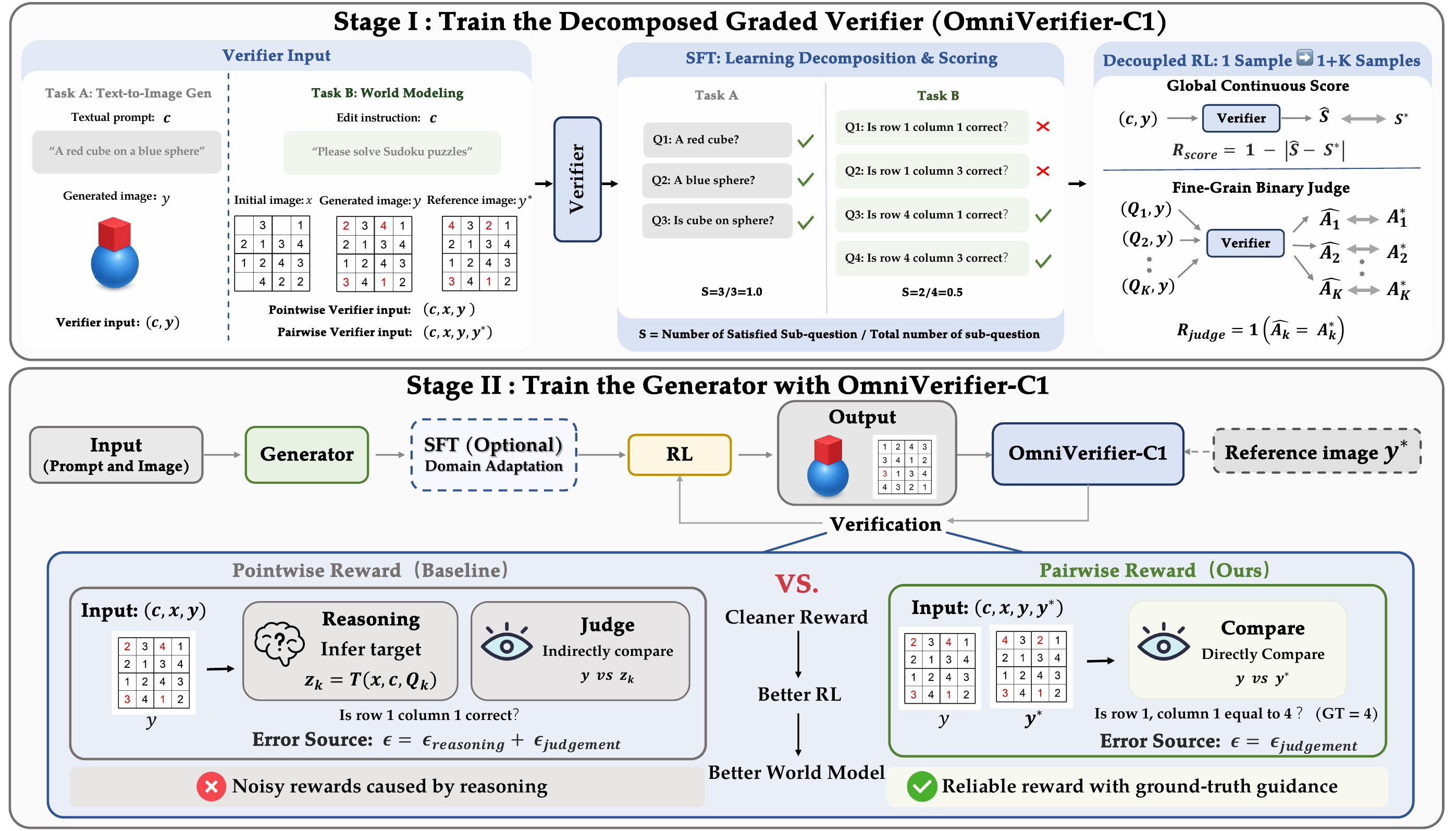}}
    \vspace{-2mm}
    \caption{\textbf{Overview of our training pipeline.}
Stage 1 trains the Decomposed Graded Verifier with decoupled reinforcement learning to produce fine-grained continuous rewards. Stage 2 introduces a pairwise judgment paradigm that provides cleaner and more reliable reward signals for generative model reinforcement learning.}
    \vspace{-2mm}
    \label{fig2_training_pipe}
\end{figure}

\paragraph{Drawbacks of Binary Reward.} 
Binary verifiers have emerged as an effective tool for improving text-to-image generation quality through inference-time scaling \citep{omniverifier}.
When they detect semantic inconsistencies between a generated image and its textual description, they produce editing instructions that guide the generative model to iteratively refine the visual output.
For this purpose, coarse binary feedback is often sufficient, as the verifier primarily serves as a gating mechanism that determines whether further refinement is needed.
However, binary verifiers face inherent limitations when used as reward models for reinforcement learning, as their feedback is insufficiently dense and discriminative to distinguish among imperfect samples of varying visual quality and thereby guide effective generative policy optimization.

\paragraph{Revisiting Decomposition-Based Continuous Rewards.} 

To address this limitation, we introduce the Decomposed Graded Verifier, a structured verification paradigm that elicits continuous scores from visual verifiers. Instead of reducing assessment to a single binary judgment, it decomposes a complex prompt into fine-grained sub-questions, evaluates each sub-question individually with binary judgments, and aggregates these local decisions into a continuous score that reflects the degree of prompt satisfaction. By preserving partial correctness in generated samples, this fine-grained reward signal provides denser and more optimization-friendly supervision for reinforcement learning.

Assume that the textual prompt $c$ can be losslessly decomposed into $K$ sub-questions $Q=\{q_1,q_2,\dots,q_K\}$. Let $z(x)\in\{0,1,\dots,K\}$ denote the number of sub-questions satisfied by a visual output $x$. We define the continuous reward and the discrete reward as $r_c(x)=\frac{z(x)}{K}$ and $r_d(x)=\mathbf{1}[z(x)=K]$ respectively. Let $q_k:=\mathbb{P}(z(x)=k)$ denote the probability that a sample satisfies exactly $k$ sub-questions, where $k=0,\dots,K$.
To quantify how well a reward recovers the true quality ordering, we define the ranking accuracy as
\begin{equation}
\mathrm{Acc}_{\mathrm{rank}}(r):=
\mathbb{P}\big((r(x)-r(x'))(z(x)-z(x'))>0 \mid z(x)\neq z(x')\big),
\end{equation}
where $x,x' \sim \text{i.i.d. } \pi_\theta(\cdot \mid c)$. This measures the probability that the reward assigns the correct ordering to a pair of samples with different true quality levels.

For the continuous reward, $r_c(x)=\frac{z(x)}{K}$ is strictly monotonic with respect to $z(x)$. Therefore, whenever $z(x)\neq z(x')$, we have $(r_c(x)-r_c(x'))(z(x)-z(x'))>0$, which gives $\mathrm{Acc}_{\mathrm{rank}}(r_c)=1$.
For the discrete reward, let $A:=\{(r_d(x)-r_d(x'))(z(x)-z(x'))>0\}$ and $B:=\{z(x)\neq z(x')\}$. Then $\mathrm{Acc}_{\mathrm{rank}}(r_d)=\mathbb{P}(A\mid B)=\frac{\mathbb{P}(A\cap B)}{\mathbb{P}(B)}$. Since $x$ and $x'$ are independent, we have $\mathbb{P}(B)=1-\sum_{k=0}^{K}q_k^2$. The event $A\cap B$ occurs only when one sample satisfies all sub-questions while the other does not. Therefore, $\mathbb{P}(A\cap B)=2q_K(1-q_K)$, and thus $\mathrm{Acc}_{\mathrm{rank}}(r_d)=\frac{2q_K(1-q_K)}{1-\sum_{k=0}^{K}q_k^2}\le 1$. A complete proof is provided in Appendix~\ref{appendix-proof}. Moreover, as long as there exist two imperfect samples with different true quality levels, i.e., there exist $0 \le k_1\neq k_2<K$, we have $\mathrm{Acc}_{\mathrm{rank}}(r_d)<1$.Therefore, $\mathrm{Acc}_{\mathrm{rank}}(r_c)=1\ge \mathrm{Acc}_{\mathrm{rank}}(r_d)$.

Therefore, graded rewards preserve the true quality ordering more faithfully by assigning larger advantage weights to samples that satisfy more sub-questions, thereby better aligning optimization with improvements in true sample quality. By contrast, discrete rewards collapse partially correct samples into the same value, failing to distinguish imperfect outputs of differing true visual quality and potentially destabilizing generative policy optimization, especially under complex compositional requirements.

\subsection{Graded Verification Enhances Multimodal Reasoning and Critique Capabilities} \label{sec 3.2}

\paragraph{Drawbacks of Binary Outcome Supervision.} Binary verifiers supervise model training solely through the final True/False outcome, implicitly assuming that a correct verdict reflects a correct reasoning process. This assumption is often reasonable in domains such as mathematical problem solving, where the solution space is large and arriving at the correct answer typically requires coherent multi-step reasoning. However, it becomes much more fragile in binary multimodal critique tasks. Without explicit supervision over the judgment process, a verifier may exploit shallow visual patterns or shortcut cues to produce seemingly correct verdicts without genuinely understanding the underlying multimodal reasoning process \citep{MetaRM,deepseekmathv2}, thereby leading to reward hacking

\begin{table}[t]
\centering
\caption{
  Performance on ViVerBench of the Binary Verifier and  Graded Verifier.
}
\vspace{-2mm}
\begin{threeparttable}
\resizebox{\linewidth}{!}{
\renewcommand{\arraystretch}{1.3}
\begin{tabular}{l|ccc|cc|cc|ccc|cccc|cc|c}
\toprule
\multirow{2}{*}{\textbf{Methods}} 
& \multicolumn{3}{c}{\textbf{Concept Existence}} 
& \multicolumn{2}{c}{\textbf{Object Relation}} 
& \multicolumn{2}{c}{\textbf{World Dynamics}} 
& \multicolumn{3}{c}{\textbf{Image Annotation}} 
& \multicolumn{4}{c}{\textbf{State Value Evaluation}} 
& \multicolumn{2}{c}{\textbf{STEM}} 
& \multirow{2}{*}{\textbf{Overall}} \\
\cmidrule(lr){2-4} 
\cmidrule(lr){5-6} 
\cmidrule(lr){7-8} 
\cmidrule(lr){9-11} 
\cmidrule(lr){12-15} 
\cmidrule(lr){16-17}
& Obj. & Attr. & Abs.P. 
& Spat. & N-Spat. 
& S-Phy & D-Phy 
& BBox & Point & Count 
& Maze & F.Lake & Robot. & GUI 
& Chart & LaTeX 
& \\ 
\midrule

Qwen3-VL-8B 
& 0.710 & 0.690 & 0.562 
& 0.642 & 0.716 
& 0.604 & 0.593 
& 0.831 & 0.741 & 0.489 
& 0.420 & 0.693 & 0.671 & 0.787 
& 0.540 & 0.773 
& 0.654 \\

\ \ + RL(Binary)
& 0.732 & 0.727 & 0.579 
& 0.713 & 0.754 
& 0.600 & 0.596 
& 0.843 & 0.770 & 0.566 
& 0.434 & 0.674 & 0.645 & 0.782 
& 0.526 & 0.795 
& 0.671 \\

\ \ + SFT(Binary)
& 0.689 & 0.695 & 0.620 
& 0.654 & 0.703 
& 0.546 & 0.627 
& 0.802 & 0.735 & 0.504 
& 0.414 & 0.714 & 0.661 & 0.792 
& 0.538 & 0.802 
& 0.656 \\

\ \ +SFT(Binary) + RL(Binary)
& 0.735 & 0.721 & 0.554 
& 0.719 & 0.748 
& 0.521 & 0.636 
& 0.840 & 0.765 & 0.532 
& 0.446 & 0.687 & 0.682 & 0.799 
& 0.553 & 0.814 
& 0.672 \\

\ \ +SFT(Graded)
& 0.724 & 0.709 & 0.535 
& 0.679 & 0.657 
& 0.613 & 0.596 
& 0.817 & 0.752 & 0.543 
& 0.427 & 0.661 & 0.576 & 0.769 
& 0.656 & 0.830 
& 0.659 \\

\rowcolor{gray!10} 
\ \ \textbf{+SFT(Graded)+RL(Graded)}
& 0.758 & 0.735 & 0.541 
& 0.736 & 0.766 
& 0.573 & 0.582 
& 0.859 & 0.781 & 0.559 
& 0.466 & 0.679 & 0.693 & 0.810 
& 0.567 & 0.839 
& 0.684 \\

\bottomrule
\end{tabular}
}
\end{threeparttable}
\vspace{-2mm}
\label{ViVerBench}
\end{table}

\vspace{-2mm}
\paragraph{Experiment Setup.}
We collect textual prompts from ShareGPT-4o-Image  \citep{chen2025sharegpt4oimg} and control prompt length to construct samples with varying levels of complexity. To enhance visual diversity, we randomly sample from multiple generative models, including FLUX.1-dev \citep{labs2025flux}, Qwen-Image \citep{qwen-image-edit2509}, and Z-Image \citep{team2025zimage}, to generate the corresponding images. We then perform supervised fine-tuning (SFT) to learn the sub-question decomposition paradigm distilled from Seed1.8 \citep{seed2026seed1.8}.
 Using the SFT-trained model, we filter out samples whose predictions exactly match the ground truth, and further apply a GPT-5 \citep{gpt5_openai} Bo10 filtering strategy to remove samples with accuracy below 60\% while controlling the proportion of fully correct samples at 30\%.
The remaining high-quality samples are used for \textit{Decoupled Reinforcement Learning}  with DAPO \citep{yu2025dapo} on Qwen3-VL-8B \citep{qwen3}. The model is trained for 80 steps on 16 NVIDIA A800-80G GPUs and evaluated on ViVerBench \citep{omniverifier}.

\vspace{-2mm}
\paragraph{Decoupled Reinforcement Learning.} As shown in Fig.~\ref{fig2_training_pipe}, given a textual prompt $c$ and its corresponding generated image $y$, we convert each pair into $K+1$ reinforcement learning samples, where $K$ denotes the number of decomposed sub-questions. Specifically, one sample is used for decomposition-based continuous scoring policy optimization, while the remaining $K$ samples are used for binary fine-grained judgment policy optimization.
For the decomposition-based scoring sample, the verifier takes the textual prompt $c$ and the generated image $y$ as input, and is trained to predict the global alignment score induced by the decomposition. Given the predicted score $\hat{S}\in[0,1]$ and the ground-truth score $S^*\in[0,1]$, the reward is defined as $\mathcal{R}_{\mathrm{score}} = 1 - |\hat{S} - S^*|$. For each fine-grained judgment sample, the verifier takes the $k$-th sub-question $Q_k$ and the generated image $y$ as input, and is trained to predict the binary answer associated with this sub-question. Given the predicted answer $\hat{A}_k\in\{0,1\}$ and the ground-truth answer $A_k^*\in\{0,1\}$, the reward is defined as $\mathcal{R}_{\mathrm{judgment}} = \mathbbm{1}[\hat{A}_k = A_k^*]$.
During reinforcement learning, we mix these two types of samples and optimize them with their respective rewards, forming a decoupled RL training scheme that jointly optimizes sub-question-level binary judgment and global alignment scoring. Precise 0/1 feedback on individual sub-questions enables the verifier to develop fine-grained verification capabilities that provide a foundation for reliable continuous scoring.

\vspace{-2mm}
\paragraph{Experiment Analysis.}
Table~\ref{ViVerBench} shows that the Decomposed Graded Verifier achieves stronger generalization than the Discrete Binary Verifier on ViVerBench, a comprehensive and challenging benchmark for evaluating multimodal reasoning and critique.
Notably, the Decomposed Graded Verifier requires supervised fine-tuning (SFT) to learn the structured decomposition paradigm. By contrast, the Discrete Binary Verifier only needs to produce a binary True/False verdict with a corresponding textual explanation, and can therefore be directly optimized through reinforcement learning.
For a fair comparison, we evaluate the Discrete Binary Verifier in both RL-only and SFT+RL settings. The results show that the Discrete Binary Verifier consistently underperforms the Decomposed Graded Verifier, regardless of whether it is first adapted with SFT.
This indicates that the key factor is not the training pipeline itself, but the nature of the supervision signal. Binary supervision is highly compressed and sparse, which encourages the model to rely on shallow patterns rather than capturing complex semantic structures. In contrast, the Decomposed Graded Verifier decomposes complex tasks into multiple sub-questions and performs fine-grained judgments, thereby explicitly enforcing a structured reasoning process. This mechanism provides richer and more informative feedback signals, effectively mitigating reward hacking and substantially improving the model's multimodal reasoning and critique capabilities. 
Moreover, our \textbf{Decoupled RL training strategy} jointly optimizes sub-question-level binary judgment and global alignment scoring, using precise 0/1 feedback on individual sub-questions to strengthen fine-grained verification capabilities and thereby provide a foundation for robust generalization to out-of-distribution (OOD) scenarios.

\section{Verifier-Guided Reinforcement Learning for Generative Models} \label{sec 4}
In this section, we investigate how the Decomposed Graded Verifier guides generative model optimization through reinforcement learning. We begin with traditional text-to-image generation task, and then focus on studying how to adapt the verifier to provide more reliable reward signals for rule-governed world modeling scenarios.

\begin{figure}[t]
    \centerline{\includegraphics[width=\textwidth]{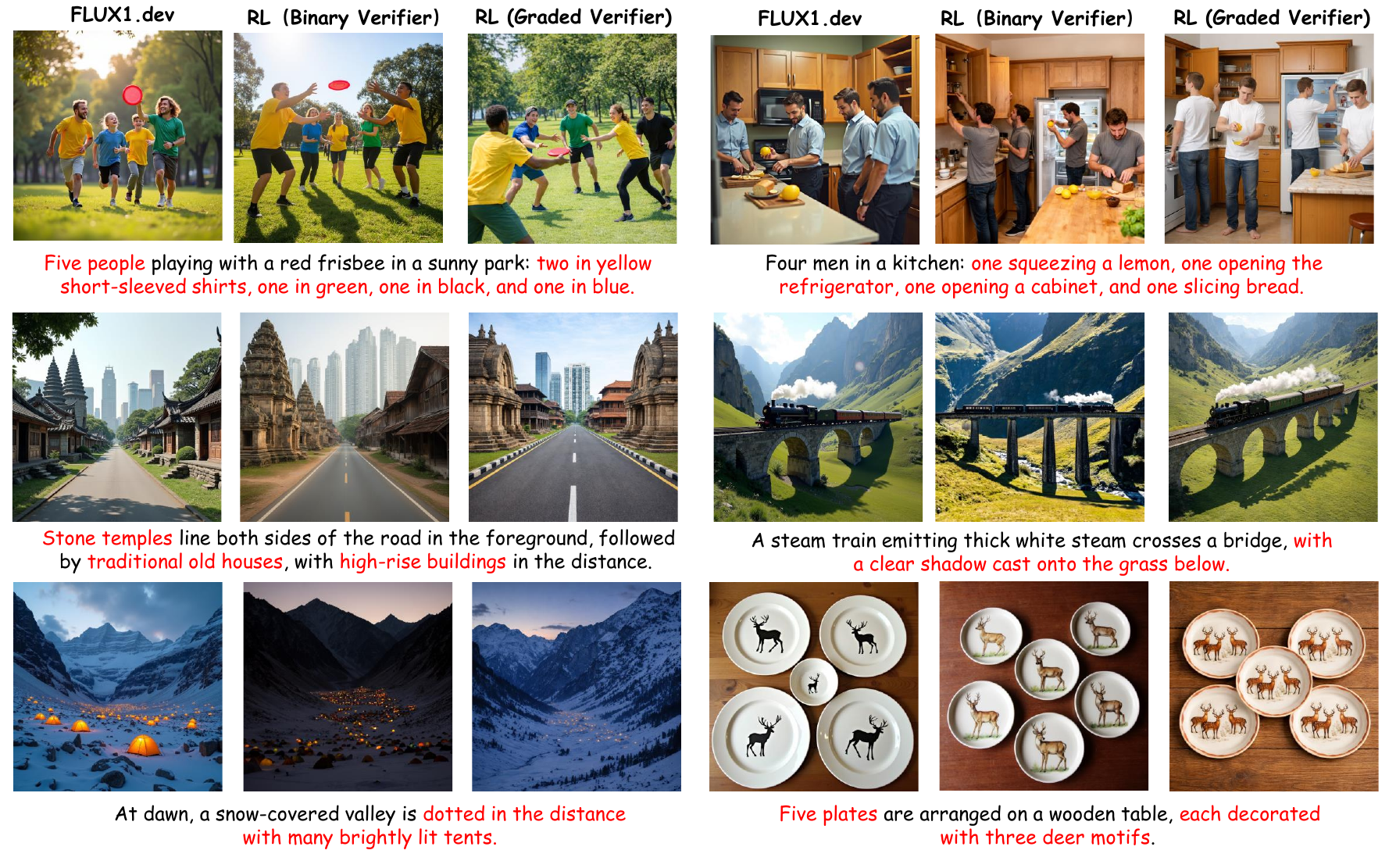}}
        \vspace{-3mm}
    \caption{Qualitative visualization of traditional text-to-image generation.}
    \vspace{-2mm}
    \label{traditional_t2i}
\end{figure}

\subsection{Verifier-Guided RL for Traditional Text-to-Image Generation}
Given a textual prompt $c$ and its generated image $y$, the Decomposed Graded Verifier first decomposes $c$ into a set of fine-grained sub-questions $Q=\{q_1,q_2,\dots,q_K\}$ and then produces a continuous score by aggregating their verification results:
$S = \frac{1}{K}\sum_{k=1}^{K}\mathcal{J}(q_k, y)$,
where $\mathcal{J}(q_k,y)\in\{0,1\}$ denotes a binary judgment function that evaluates whether the generated image $y$ satisfies sub-question $q_k$. This score is used as the reward signal for reinforcement learning to guide policy optimization.

\subsection{Verifier-Guided RL for Rule-Governed World Modeling} \label{sec4.2}

Traditional text-to-image generation primarily targets semantic alignment, where a prompt may admit multiple diverse and visually plausible outputs, leaving no well-defined unique ground-truth image.  By contrast, world simulation and reasoning tasks are rule-governed and highly structured, where the desired final state is often uniquely determined by the initial state and the underlying rules, thereby making a ground-truth image available in many scenarios.  Motivated by this property, we explicitly introduce the ground-truth image as an additional visual input to the verifier, which simplifies the verification process and enables the verifier to produce more accurate and reliable reward signals for rule-governed generative world modeling.

Formally, given an initial visual state $x$ and a rule-based instruction $c$, the generative model $\pi_\theta$ is required to \textbf{understand the underlying rules and perform rule-governed reasoning} to produce a resulting image $y \sim \pi_\theta(y \mid x,c)$ , which can be viewed as a visual prediction of the target final state. Similar to the traditional text-to-image setting, the instruction $c$ can be decomposed into $K$ fine-grained sub-questions $Q=\{q_1,q_2,\dots,q_K\}$, each assessing whether the resulting image $y$ satisfies a specific sub-rule or achieves a corresponding sub-goal.

\vspace{-3mm}
\paragraph{Definition1: Pointwise Verification.}
As shown in Fig.~\ref{fig2_training_pipe}, the pointwise verifier takes $(x,c,y)$ as input. For each sub-question $q_k$, it must first infer the corresponding target state from the initial visual state $x$ and the underlying rules $c$, denoted as $z_k=\mathcal{T}(q_k,x,c)$. The verifier then assesses whether the resulting image $y$ is consistent with this internally inferred target state $z_k$. 
Accordingly, the pointwise verification score is defined as $S_{\mathrm{point}}=\frac{1}{K}\sum_{k=1}^{K}\mathcal{J}(q_k,y,z_k)$, and its deviation from the true score consists of both reasoning error and judgment error: $\epsilon_{\mathrm{point}}=\epsilon_{\mathrm{reasoning}}+\epsilon_{\mathrm{judge}}$.

\vspace{-3mm}
\paragraph{Definition2: Pairwise Verification.}
In contrast, pairwise verification explicitly incorporates the ground-truth image $y^*$ as an additional visual input, with the verifier taking $(x,c,y,y^*)$ as input. For each sub-question $q_k$, the verifier compares the generated image $y$ against the ground-truth image $y^*$ and assigns a corresponding score, hence the name pairwise verification. By avoiding internal target-state inference, the pairwise verification score is defined as $S_{\mathrm{pair}}=\frac{1}{K}\sum_{k=1}^{K}\mathcal{J}(q_k,y,y^*)$, and its deviation from the true score is reduced to the judgment error: $\epsilon_{\mathrm{pair}}=\epsilon_{\mathrm{judge}}$.

By simplifying the reason-then-judge process into direct reference-based comparison and thereby reducing the reasoning errors associated with target-state inference, pairwise verification provides a substantially more reliable reward mechanism than pointwise verification in rule-governed world-modeling scenarios. This property is crucial for post-training generative models, where the effectiveness of reinforcement learning hinges on the accuracy and stability of the reward signal.

\section{Experiments} \label{sec 5}
In this section, we conduct comprehensive experiments to demonstrate the advantages of the Decomposed Graded Verifier as a generative reward model and the superiority of the pairwise verification mechanism in rule-governed world modeling scenarios.

\begin{table*}[t!]
\begin{center}
\caption{Performance on GenEval and CLIPScore under different reward models.}
\vspace{-3mm}
\label{tab:geneval}
\resizebox{1\textwidth}{!}{
\begin{tabular}{l | c c c c c c c | c}
\toprule
\multirow{2}{*}{\textbf{Methods}} 
& \multicolumn{7}{c|}{\textbf{GenEval}} 
& \multirow{2}{*}{\textbf{CLIPScore}} \\
& Single Obj. & Two Obj. & Counting & Colors & Position & Attri. Bind. & Overall & \\
\midrule

Stable-Diffusion-3.5-medium                                    
& 0.99 & 0.78 & 0.69 & 0.79 & 0.25 & 0.45 & 0.66 & 32.8 \\

\ \ +RL(UnifiedReward)              
& 0.99 & 0.84 & 0.75 & 0.86 & 0.36 & 0.55 & 0.73 & 33.7 \\

\ \ +RL(Discrete Binary Verifier)        
& 0.99 & 0.80 & 0.72 & 0.81 & 0.29 & 0.47 & 0.68 & 33.1 \\

\rowcolor{gray!10} 
\ \ \textbf{+RL(Decomposed Graded Verifier)}
& \textbf{1.00} & \textbf{0.90} & \textbf{0.79} & \textbf{0.89} 
& \textbf{0.41} & \textbf{0.63} & \textbf{0.77} & \textbf{34.4} \\

\midrule

FLUX.1-dev                                    
& 0.99 & 0.85 & 0.61 & 0.84 & 0.22 & 0.58 & 0.68 & 33.2 \\

\ \ +RL(UnifiedReward)                    
& 1.00 & 0.90 & 0.65 & 0.89 & 0.33 & 0.64 & 0.74 & 34.0 \\

\ \ +RL(Discrete Binary Verifier)        
& 1.00 & 0.88 & 0.63 & 0.85 & 0.25 & 0.59 & 0.70 & 33.3 \\

\rowcolor{gray!10} 
\ \ \textbf{+RL(Decomposed Graded Verifier)} 
& \textbf{1.00} & \textbf{0.93} & \textbf{0.77} & \textbf{0.90} 
& \textbf{0.43} & \textbf{0.71} & \textbf{0.79} & \textbf{34.5} \\

\bottomrule
\end{tabular}
}
\end{center}
\vspace{-2mm}
\end{table*}

\subsection{Experiment Setup}

\paragraph{Training Datasets.}
For traditional text-to-image tasks, we use textual prompts from ShareGPT-4o-Image~\citep{chen2025sharegpt4oimg}, the same data source used for verifier training.
For rule-governed world modeling scenarios, we construct six game-based tasks to comprehensively evaluate generative models' abilities in \textbf{path planning, logical reasoning, geometric construction, pattern induction, character recognition, and spatial reasoning} .
\begin{itemize}
     \item \textbf{Maze \& Maze-Hard.} Find a valid path from the red start circle to the blue end cross without touching or crossing any black walls. Maze-Hard contains larger-scale maze instances.
     
     \item \textbf{Sudoku \& Sudoku-Hard.} Fill the blanks in a $4\times4$ Sudoku grid so that each row, column, and $2\times2$ subgrid contains $\{1,2,3,4\}$ exactly once. Sudoku-Hard contains more blanks.
     \item \textbf{Auxiliary Line Construction \& Auxiliary Line Construction-Hard.} Draw auxiliary lines in a geometric figure. The hard setting requires multiple compositional constructions.
  
     \item \textbf{Sequence Prediction \& Sequence Prediction-Hard.} Predict the next geometric figure from  shape or color patterns. The hard setting requires simultaneous color-shape changes.
     
    \item \textbf{Word Search \& Word Search-Hard.} Locate the target word in an $N\times N$ grid, where each row contains one character in the correct order. The hard setting corresponds to a larger $N$.
     
     \item \textbf{3D Projection \& 3D Projection-Hard.} Generate the specified 2D orthographic view of a 3D structure, such as the top, front, or side view. The hard setting uses larger 3D structures.
\end{itemize}

\paragraph{Training Details.}
For traditional text-to-image tasks, we adopt Stable-Diffusion-3.5-medium and FLUX.1-dev \citep{labs2025flux} as the backbone generative models. For rule-governed world-modeling scenarios, we use Qwen-Image-Edit-2509 \citep{qwen-image-edit2509} as the backbone. All models are trained for 200 steps using the DiffusionNFT algorithm \citep{zheng2025diffusionnft} on 16 NVIDIA A800-80G GPUs.

\subsection{Text-to-Image Generation}

\paragraph{Graded Rewards Are Better than Binary Rewards.}
Table~\ref{tab:geneval} compares the Decomposed Graded Verifier with the Discrete Binary Verifier as reinforcement learning reward models across different generative backbones, including Stable-Diffusion-3.5-Medium and FLUX.1-dev. The Decomposed Graded Verifier consistently achieves superior performance, demonstrating that fine-grained graded reward signals are more effective for improving the semantic alignment capability of generative models. Fig.~\ref{traditional_t2i} further supports this conclusion through qualitative visualizations.
Fig.~\ref{fig1}  (a)(b) shows the training dynamics under the two reward models. 
We evaluate intermediate checkpoints every 10 steps on a held-out set of 500 textual prompts under the same decomposition-based scoring protocol, \textbf{using Gemini3.1-Pro~\citep{googledeepmind2026gemini31pro} as an evaluator not involved in data processing}.
Compared with the Discrete Binary Verifier, training with the Decomposed Graded Verifier yields steadier and more pronounced performance improvements, suggesting that denser and more discriminative rewards support more stable and effective optimization.

\begin{figure}[t]
    \centerline{\includegraphics[width=\textwidth]{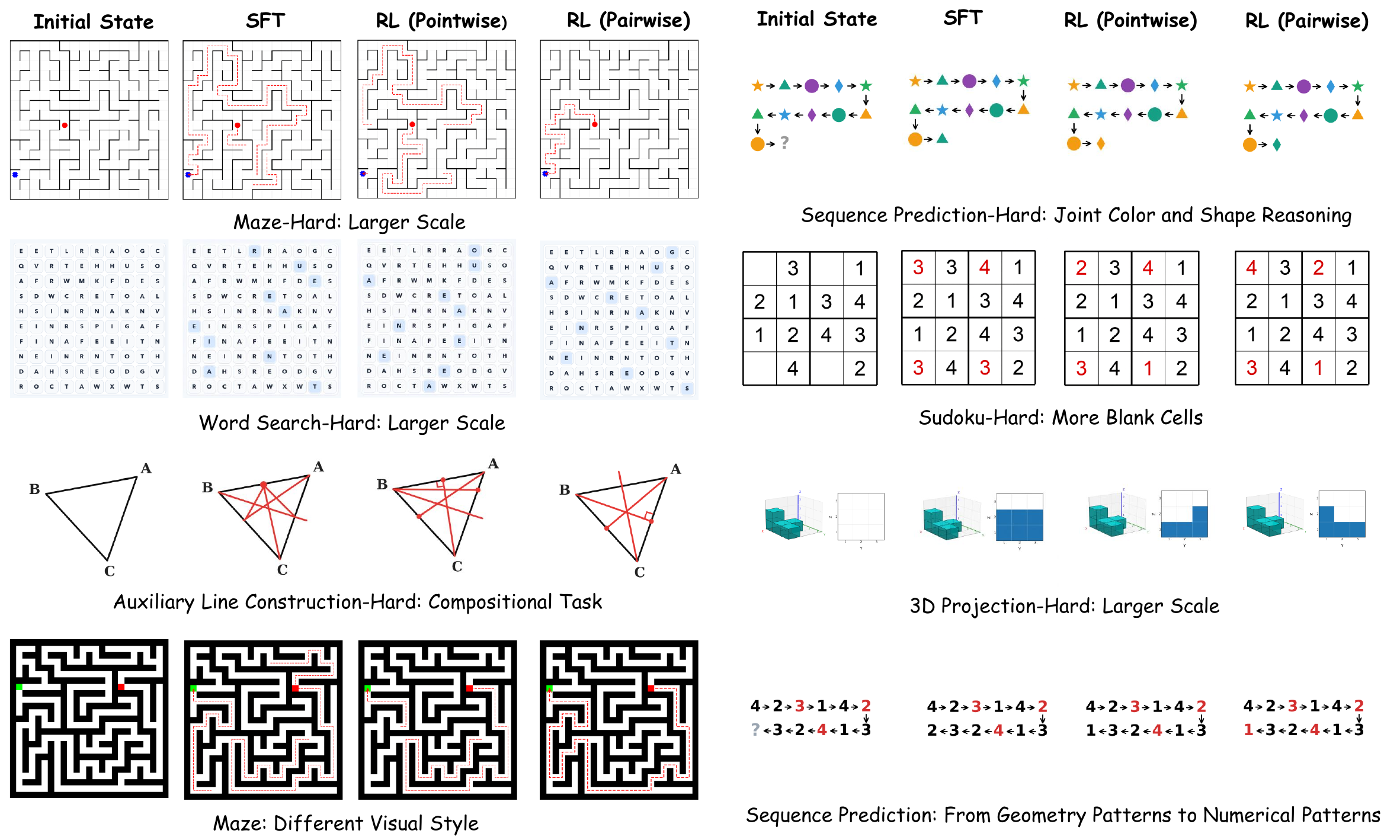}}
    \vspace{-2mm}
    \caption{Qualitative visualization of rule-governed world-modeling tasks.}
    \vspace{-2mm}
    \label{fig_reasoning}
\end{figure}

\paragraph{Comparison with Other Continuous Reward Models.}
We compare the Decomposed Graded Verifier with UnifiedReward \citep{unifiedreward}, which also focuses on semantic alignment and produces continuous reward scores. As shown in Table~\ref{tab:geneval}, our method achieves superior performance on both GenEval \citep{ghosh2023geneval} and CLIPScore \citep{hessel2021clipscore}, supporting the ability of the Decomposed Graded Verifier trained with our Decoupled RL strategy to provide more accurate and fine-grained reward signals and highlighting its advantages as a generative reward model.

\subsection{Rule-Governed Generative World Modeling}

\paragraph{Advantages of Pairwise Verification over Pointwise Verification.}

As shown in Tables~\ref{tab:game_editing} and~\ref{tab:rise_uniredit}, pairwise verification consistently outperforms pointwise verification on RISEBench~\citep{zhao2025risebench}, UniREditBench~\citep{han2025unireditbench}, and our constructed game-based benchmark. These results indicate that incorporating the ground-truth image as an explicit reference significantly enhances the accuracy and reliability of reward signals.
The training dynamics in Fig.~\ref{fig1}(e) summarize average performance across our six constructed game-based tasks and further highlight the optimization advantages of pairwise verification over pointwise verification. We evaluate intermediate checkpoints every 10 training steps on the held-out validation sets. \textbf{To improve evaluation reliability, we use Gemini3.1-Pro~\citep{googledeepmind2026gemini31pro} as an evaluator not involved in data processing and explicitly provide the ground-truth image following the pairwise verification protocol.}

The pairwise-trained model exhibits steadier and more pronounced performance improvements, whereas the pointwise-trained model shows clear performance fluctuations. Pointwise rewards require the verifier to infer an implicit target state, which can introduce substantial noise into the feedback in complex reasoning scenarios. In contrast, pairwise verification directly compares the generated result with the ground-truth image, reducing reliance on target-state inference and enabling reward signals to better reflect the true visual quality. This provides more reliable supervision for policy optimization and highlights the value of  pairwise verification for developing generative models into rule-governed world-modeling agents.

\begin{table}[t]
\centering
\caption{
Performance comparison on game-based image-to-image tasks for world modeling.
}
\vspace{-2mm}
\begin{threeparttable}
\renewcommand{\arraystretch}{1.15}
\resizebox{0.99\linewidth}{!}{
\begin{tabular}{l|cc|cc|cc|cc|cc|cc}
\toprule
\multirow{2}{*}{\textbf{Methods}}
& \multicolumn{2}{c|}{\textbf{Maze}}
& \multicolumn{2}{c|}{\textbf{Sudoku}}
& \multicolumn{2}{c|}{\makecell{\textbf{Aux. Line}\\\textbf{Construction}}}
& \multicolumn{2}{c|}{\makecell{\textbf{Sequence}\\\textbf{Prediction}}}
& \multicolumn{2}{c|}{\makecell{\textbf{Word}\\\textbf{Search}}}
& \multicolumn{2}{c}{\makecell{\textbf{3D}\\\textbf{Projection}}} \\
\cmidrule(lr){2-3}
\cmidrule(lr){4-5}
\cmidrule(lr){6-7}
\cmidrule(lr){8-9}
\cmidrule(lr){10-11}
\cmidrule(lr){12-13}
& Normal & Hard
& Normal & Hard
& Normal & Hard
& Normal & Hard
& Normal & Hard
& Normal & Hard \\
\midrule

Qwen-Image-Edit-2509
& 17.31 & 8.73
& 22.21 & 5.62
& 16.79 & 7.32
& 9.18 & 4.33
& 6.28 & 2.46
& 3.86 & 1.71 \\

\ \ + SFT
& 56.77 & 18.96
& 63.62 & 29.99
& 57.62 & 26.84
& 46.53 & 27.96
& 37.36 & 17.29
& 42.94 & 23.37 \\

\ \ +SFT+RL(Pointwise)
& 67.81 & 34.78
& 77.74 & 50.71
& 66.98 & 37.49
& 61.27 & 34.18
& 50.92 & 26.94
& 52.19 & 33.16 \\

\rowcolor{gray!10} 
\ \ +SFT+RL(Pairwise)
& \textbf{85.31} & \textbf{51.92}
& \textbf{89.97} & \textbf{62.99}
& \textbf{76.19} & \textbf{46.12}
& \textbf{73.48} & \textbf{63.64}
& \textbf{62.18} & \textbf{34.27}
& \textbf{69.28} & \textbf{47.51} \\

\bottomrule
\end{tabular}
}
\end{threeparttable}
\vspace{-1mm}
\label{tab:game_editing}
\end{table}

\begin{table*}[t!]
\centering
\caption{Performance comparison on RISEBench and UniREditBench.}
\vspace{-2mm}
\label{tab:rise_uniredit}
\begin{threeparttable}
\renewcommand{\arraystretch}{1.35}
\resizebox{\textwidth}{!}{
\begin{tabular}{l|ccccc||ccccc|ccccc|c}
\toprule
\multirow{2}{*}{\textbf{Methods}}
& \multicolumn{5}{c||}{\textbf{RISEBench}}
& \multicolumn{5}{c|}{\textbf{UniREdit-Real}}
& \multicolumn{5}{c|}{\textbf{UniREdit-Game}}
& \multirow{2}{*}{\textbf{Overall}} \\
\cmidrule(lr){2-6} \cmidrule(lr){7-11} \cmidrule(lr){12-16}
& \textbf{T} & \textbf{C} & \textbf{S} & \textbf{L} & \textbf{Overall}
& \textbf{AM} & \textbf{ST} & \textbf{PI} & \textbf{PR} & \textbf{Avg.}
& \textbf{SI} & \textbf{SR} & \textbf{LHP} & \textbf{LPS} & \textbf{Avg.}
& \\
\midrule

Qwen-Image-Edit-2509
& 2.4 & 7.8 & 14.0 & 1.2 & 6.7
& 70.04 & 68.49 & 67.26 & 64.11 & 67.48
& 52.31 & 32.96 & 42.46 & 37.89 & 41.41
& 53.16 \\

\ \ + SFT
& 3.5 & 7.8 & 16.0 & 5.9 & 8.6
& 70.56 & 68.68 & 67.55 & 64.98 & 67.94
& 66.92 & 47.48 & 56.06 & 43.98 & 53.61
& 59.14 \\

\ \ + SFT + RL(Pointwise)
& 5.9 & 8.9 & 18.0 & 12.9 & 11.7
& 71.24 & 68.73 & 68.38 & 65.77 & 68.53
& 68.73 & 49.16 & 60.32 & 56.41 & 58.66
& 62.37 \\

\rowcolor{gray!10}
\ \ + SFT + RL(Pairwise)
& \textbf{7.1} & \textbf{10.0} & \textbf{19.0} & \textbf{23.5} & \textbf{15.0}
& \textbf{71.87} & \textbf{69.02} & \textbf{68.72} & \textbf{66.08} & \textbf{68.92}
& \textbf{70.23} & \textbf{51.33} & \textbf{68.93} & \textbf{63.53} & \textbf{63.51}
& \textbf{65.15} \\

\bottomrule
\end{tabular}
}

\vspace{-1mm}
\end{threeparttable}
\vspace{-2mm}
\end{table*}

\vspace{-1mm}
\paragraph{From Easy-to-Hard and Domain Generalization.}
To more comprehensively evaluate the generalization potential of the model trained with pairwise verification, we construct a more challenging Hard Dataset that is excluded from training. This dataset contains larger-scale instances and compositional tasks designed to test reasoning in more complex structured environments.
As shown in Table~\ref{tab:game_editing} and Fig.~\ref{fig_reasoning}, the model exhibits strong performance across out-of-distribution scenarios, including shifts in the visual domain. These results suggest that the model acquires transferable capabilities for structured reasoning and rule-governed world modeling beyond task-specific training patterns, highlighting the advantages of pairwise verification in supporting robust generalization.

\vspace{-1mm}
\paragraph{SFT Provides Domain-Specific Execution Priors.}
Tables~\ref{tab:game_editing} and~\ref{tab:rise_uniredit} show that SFT yields substantial performance gains, which can be primarily attributed to task-domain adaptation.
Since Qwen-Image-Edit-2509 has not been specifically trained on game-based image-to-image generation tasks, it often fails to interpret the required transformation. In practice, the model tends to either the model tends to either leave the input image unchanged or distort its original structure, indicating that it lacks the task-specific action priors required for this domain.
After supervised fine-tuning on our constructed game-based dataset, the generative model learns the visual style of these tasks and the corresponding rule-based transformation patterns. Although the model still often fails to produce fully correct solutions, it begins to perform meaningful operations toward the desired final state rather than generating irrelevant or destructive visual outputs, leading to clear performance improvements. These results indicate that \textbf{SFT primarily equips the model with domain-specific execution ability}, while subsequent reinforcement learning further endows the generative model with reasoning intelligence.

\paragraph{Generative Intelligence Beyond the Text Encoder.}

In both multimodal understanding and generation models, the text encoder is often regarded as the primary source of reasoning intelligence. For generative models, this view implies that the text encoder interprets instructions and provides semantic guidance, while the DiT primarily executes the visual generation process. However, in our training pipeline, the text encoder remains frozen throughout both SFT and RL, and only the DiT parameters are updated. The reasoning capabilities observed after training suggest that reasoning intelligence is not encoded exclusively in the text encoder, as the DiT can also learn, internalize, and express capabilities for structured reasoning and rule-governed world modeling.

\section{Conclusion}
This work introduces the Decomposed Graded Verifier as an effective reward model for improving visual generation in both traditional text-to-image tasks and structured world-modeling scenarios. Our contributions are twofold. First, we propose a Decoupled RL training strategy that jointly optimizes sub-question-level binary judgment and global alignment scoring, using precise local feedback to develop fine-grained verification capabilities that support accurate continuous scoring. Empirical results further demonstrate improved out-of-domain multimodal critic capabilities. Second, focusing on rule-governed world modeling, we introduce a pairwise verification paradigm that explicitly incorporates ground-truth images as visual references. This simplifies verification by shifting from a reason-then-judge process to direct reference-based comparison, thereby producing more accurate and reliable reward signals. Extensive experiments validate the effectiveness of both the Decomposed Graded Verifier trained with our Decoupled RL strategy and the pairwise verification paradigm.

\newpage

\bibliography{example_paper}
\bibliographystyle{colm2025_conference}

\newpage

\appendix

\section{Proof of Ranking Preservation}
\label{appendix-proof}

In this section we prove that Decomposed Graded Verification preserves the true quality ordering of generated samples more faithfully than Discrete Binary Verification. The argument rests on a single observation: under GRPO, the relative update strength among samples is dictated by their group-normalized advantages, whose ordering is fully inherited from the underlying reward values. It therefore suffices to compare how the two reward designs rank pairs of samples with distinct latent quality.

\paragraph{Setup.}
Let $c$ be a prompt and $x \sim \pi_\theta(\cdot \mid c)$ a sample drawn from the current policy. Suppose $c$ admits a lossless decomposition into $K$ sub-questions, and let
\begin{equation}
    z(x) \in \{0,1,\dots,K\}
\end{equation}
denote the number of sub-questions satisfied by $x$. We treat $z(x)$ as the latent fine-grained quality level of $x$ and write
\begin{equation}
    q_k := \mathbb{P}(z(x)=k), \qquad k=0,\dots,K,
\end{equation}
where the probability is taken with respect to $x\sim \pi_\theta(\cdot\mid c)$.

We compare two reward designs:
\begin{equation}
    r_c(x)=\frac{z(x)}{K},
    \qquad
    r_d(x)=\mathbf{1}[z(x)=K],
\end{equation}
namely the decomposed graded reward $r_c$ and the discrete binary reward $r_d$.

For two independent samples $x,x'\sim \pi_\theta(\cdot\mid c)$, the pairwise ranking accuracy of a reward function $r$ is defined as
\begin{equation}
\mathrm{Acc}_{\mathrm{rank}}(r)
:=
\mathbb{P}
\left(
(r(x)-r(x'))(z(x)-z(x'))>0
\mid
z(x)\neq z(x')
\right).
\label{eq:rank_acc}
\end{equation}
In words, conditional on the two samples carrying distinct true quality levels, $\mathrm{Acc}_{\mathrm{rank}}(r)$ is the probability that $r$ assigns them a strictly correct relative ordering.

\paragraph{Remark on non-degeneracy.}
Throughout the proof we assume
\begin{equation}
    \mathbb{P}(z(x)\neq z(x'))>0,
\end{equation}
i.e., the distribution of $z(x)$ is not a point mass on a single quality level. When this fails, no pair of samples carries distinct quality levels and Eq.~\eqref{eq:rank_acc} is undefined; the comparison is vacuous in this case and we exclude it without loss of generality.

\begin{lemma}[Advantage ordering is inherited from reward ordering]
\label{lem:advantage_order}
Consider a GRPO rollout group $\{x_i\}_{i=1}^n$ with rewards $r_i=r(x_i)$, and let
\begin{equation}
    A_i(r)=\frac{r_i-\bar r}{\sigma_r},
    \qquad
    \bar r=\frac{1}{n}\sum_{j=1}^n r_j,
    \qquad
    \sigma_r^2=\frac{1}{n}\sum_{j=1}^n(r_j-\bar r)^2
\end{equation}
denote the group-normalized advantage. If $\sigma_r>0$, then for any pair of samples $x_i,x_j$,
\begin{equation}
    \operatorname{sign}\big(A_i(r)-A_j(r)\big)
    =
    \operatorname{sign}\big(r_i-r_j\big).
\end{equation}
Consequently, the relative ordering of GRPO update weights coincides exactly with the relative ordering of rewards.
\end{lemma}

\begin{proof}
Subtracting the two normalized advantages cancels the common mean,
\begin{equation}
    A_i(r)-A_j(r)
    =
    \frac{(r_i-\bar r)-(r_j-\bar r)}{\sigma_r}
    =
    \frac{r_i-r_j}{\sigma_r},
\end{equation}
and dividing by the positive scalar $\sigma_r$ preserves the sign. Therefore
\begin{equation}
    \operatorname{sign}\big(A_i(r)-A_j(r)\big)
    =
    \operatorname{sign}\big(r_i-r_j\big),
\end{equation}
as claimed.
\end{proof}

\begin{lemma}[The graded reward perfectly preserves the quality ordering]
\label{lem:continuous_order}
For the decomposed graded reward
\begin{equation}
    r_c(x)=\frac{z(x)}{K},
\end{equation}
the ranking accuracy satisfies
\begin{equation}
    \mathrm{Acc}_{\mathrm{rank}}(r_c)=1.
\end{equation}
\end{lemma}

\begin{proof}
Since $r_c$ is a strictly increasing affine function of $z$, for any pair $x,x'$ with $z(x)\neq z(x')$ we have
\begin{equation}
    (r_c(x)-r_c(x'))(z(x)-z(x'))
    =
    \frac{(z(x)-z(x'))^2}{K}>0,
\end{equation}
where the strict inequality uses $K>0$ and $z(x)\neq z(x')$. Hence the event $(r_c(x)-r_c(x'))(z(x)-z(x'))>0$ holds with probability one conditional on $z(x)\neq z(x')$, and $\mathrm{Acc}_{\mathrm{rank}}(r_c)=1$.
\end{proof}

\begin{lemma}[The binary reward only ranks perfect-vs-imperfect pairs]
\label{lem:discrete_event}
For the discrete binary reward
\begin{equation}
    r_d(x)=\mathbf{1}[z(x)=K],
\end{equation}
the event
\begin{equation}
    (r_d(x)-r_d(x'))(z(x)-z(x'))>0
\end{equation}
holds if and only if exactly one of the two samples attains the perfect quality level $K$.
\end{lemma}

\begin{proof}
Conditioning on $z(x)\neq z(x')$, there are exactly three disjoint cases:

\emph{(i) Exactly one sample is perfect.} Without loss of generality, suppose $z(x)=K$ and $z(x')<K$. Then $r_d(x)-r_d(x')=1$ and $z(x)-z(x')=K-z(x')>0$, so the product is strictly positive. The case $z(x')=K, z(x)<K$ is symmetric.

\emph{(ii) Both samples are imperfect, i.e.\ $z(x)<K$ and $z(x')<K$.} Then $r_d(x)=r_d(x')=0$, so $r_d(x)-r_d(x')=0$ and the product vanishes, even though $z(x)\neq z(x')$.

\emph{(iii) Both samples are perfect.} This case violates $z(x)\neq z(x')$ and is excluded.

The product is strictly positive only in case~(i), which proves the claim.
\end{proof}

\begin{theorem}[Ranking advantage of decomposed graded verification]
\label{thm:ranking_advantage}
Assume the quality distribution $\{q_k\}_{k=0}^K$ is non-degenerate, that is,
\begin{equation}
    1-\sum_{k=0}^K q_k^2>0.
\end{equation}
Then
\begin{equation}
    \mathrm{Acc}_{\mathrm{rank}}(r_c)=1,
    \qquad
    \mathrm{Acc}_{\mathrm{rank}}(r_d)
    =
    \frac{2q_K(1-q_K)}
    {1-\sum_{k=0}^K q_k^2}.
\end{equation}
In particular,
\begin{equation}
    \mathrm{Acc}_{\mathrm{rank}}(r_c)
    \ge
    \mathrm{Acc}_{\mathrm{rank}}(r_d),
\end{equation}
with strict inequality whenever there exist two distinct imperfect quality levels with non-zero probability, i.e., $q_aq_b>0$ for some $0\le a<b<K$.
\end{theorem}

\begin{proof}
The identity $\mathrm{Acc}_{\mathrm{rank}}(r_c)=1$ is Lemma~\ref{lem:continuous_order}. It remains to evaluate $\mathrm{Acc}_{\mathrm{rank}}(r_d)$ and to compare the two accuracies.

\emph{Step 1: Closed form for $\mathrm{Acc}_{\mathrm{rank}}(r_d)$.} Define the events
\begin{equation}
    A := \left\{
    (r_d(x)-r_d(x'))(z(x)-z(x'))>0
    \right\},
    \qquad
    B := \left\{ z(x)\neq z(x') \right\},
\end{equation}
so that $\mathrm{Acc}_{\mathrm{rank}}(r_d) = \mathbb{P}(A\mid B) = \mathbb{P}(A\cap B)/\mathbb{P}(B)$.

By independence of $x$ and $x'$,
\begin{equation}
    \mathbb{P}(B)
    =
    1-\mathbb{P}(z(x)=z(x'))
    =
    1-\sum_{k=0}^K q_k^2,
\label{eq:denominator}
\end{equation}
which is strictly positive by the non-degeneracy assumption. By Lemma~\ref{lem:discrete_event}, the event $A\cap B$ is exactly the event that exactly one of the two samples attains quality level $K$. Hence
\begin{align}
    \mathbb{P}(A\cap B)
    &=
    \mathbb{P}(z(x)=K,\,z(x')<K)
    +
    \mathbb{P}(z(x)<K,\,z(x')=K)
    \\
    &=
    2q_K(1-q_K).
\label{eq:numerator}
\end{align}
Dividing Eq.~\eqref{eq:numerator} by Eq.~\eqref{eq:denominator} gives the claimed formula for $\mathrm{Acc}_{\mathrm{rank}}(r_d)$.

\emph{Step 2: $\mathrm{Acc}_{\mathrm{rank}}(r_d)\le 1$.} Expanding $(\sum_k q_k)^2 = 1$ yields the standard identity
\begin{equation}
    1-\sum_{k=0}^K q_k^2
    =
    2\sum_{0\le a<b\le K}q_aq_b,
\end{equation}
and the same expansion gives
\begin{equation}
    2q_K(1-q_K) = 2\sum_{a=0}^{K-1} q_a q_K.
\end{equation}
Splitting the unordered pairs $(a,b)$ according to whether $b=K$, we obtain the decomposition
\begin{equation}
    \underbrace{1-\sum_{k=0}^K q_k^2}_{\text{denominator}}
    =
    \underbrace{2\sum_{a=0}^{K-1}q_aq_K}_{\text{numerator}}
    \;+\;
    2\sum_{0\le a<b<K}q_aq_b.
\label{eq:decomp}
\end{equation}
The second term on the right is non-negative, so the denominator is at least as large as the numerator and $\mathrm{Acc}_{\mathrm{rank}}(r_d)\le 1$. Combined with Lemma~\ref{lem:continuous_order}, this gives
\begin{equation}
    \mathrm{Acc}_{\mathrm{rank}}(r_c)
    =1\ge
    \mathrm{Acc}_{\mathrm{rank}}(r_d).
\end{equation}

\emph{Step 3: Strictness.} From Eq.~\eqref{eq:decomp}, the inequality is strict if and only if
\begin{equation}
    \sum_{0\le a<b<K}q_aq_b>0,
\end{equation}
which holds precisely when $q_aq_b>0$ for some $0\le a<b<K$, i.e., when at least two distinct imperfect quality levels are visited with positive probability. In this case,
\begin{equation}
    \mathrm{Acc}_{\mathrm{rank}}(r_d)<1
    =
    \mathrm{Acc}_{\mathrm{rank}}(r_c),
\end{equation}
which completes the proof.
\end{proof}

\begin{corollary}[Collapse of imperfect samples under binary verification]
\label{cor:collapse}
The ranking error of the discrete binary reward is exactly
\begin{equation}
    1-\mathrm{Acc}_{\mathrm{rank}}(r_d)
    =
    \frac{
    2\sum_{0\le a<b<K}q_aq_b
    }
    {
    1-\sum_{k=0}^K q_k^2
    }.
\end{equation}
Equivalently, every ranking error incurred by the binary reward comes from a pair of imperfect samples with different true quality levels.
\end{corollary}

\begin{proof}
Using the decomposition in Eq.~\eqref{eq:decomp},
\begin{align}
    1-\mathrm{Acc}_{\mathrm{rank}}(r_d)
    &=
    \frac{
    \left(1-\sum_{k=0}^K q_k^2\right) - 2\sum_{a=0}^{K-1}q_aq_K
    }
    {
    1-\sum_{k=0}^K q_k^2
    }
    =
    \frac{
    2\sum_{0\le a<b<K}q_aq_b
    }
    {
    1-\sum_{k=0}^K q_k^2
    }.
\end{align}
The numerator counts (twice) the probability mass of unordered pairs of distinct \emph{imperfect} quality levels, which by Lemma~\ref{lem:discrete_event} are precisely the pairs $r_d$ fails to rank. This proves the corollary.
\end{proof}

\paragraph{Implication for GRPO.}
By Lemma~\ref{lem:advantage_order}, GRPO assigns larger update weights to samples with larger rewards. Therefore, a reward model that better preserves the true quality ordering also provides more reliable relative policy updates. The decomposed graded reward is strictly monotonic with respect to the number of satisfied sub-questions and thus preserves all pairwise quality comparisons. In contrast, the binary reward collapses all imperfect samples into the same reward, making GRPO unable to distinguish partially correct samples from severely incorrect ones. Consequently, Decomposed Graded Verification provides a denser and more order-preserving reward signal for GRPO.

\section{Other Results}
\subsection{Training Dynamics across Six Game-Based Tasks}
\label{six_data}

Fig.~\ref{fig:six_task_training} presents the training dynamics of pairwise and pointwise verification across our six constructed rule-governed tasks. Following the evaluation protocol in the main text, we evaluate intermediate checkpoints every 10 training steps on the corresponding validation sets, using Gemini3.1-Pro~\citep{googledeepmind2026gemini31pro} as an evaluator not involved in data processing and providing ground-truth images as explicit references.

Across all six tasks, pairwise verification generally yields more sustained performance improvements and higher final scores, whereas pointwise verification exhibits more pronounced fluctuations and limited progress. This trend is consistent with the aggregate results in the main text. By directly comparing generated outputs with ground-truth references, pairwise verification reduces reliance on internal target-state inference and can provide more reliable reward signals for policy optimization. These results further support the effectiveness of reference-based verification across different types of rule-governed visual generation tasks.

\begin{figure}[h]
    \centering
    \includegraphics[width=.95\textwidth]{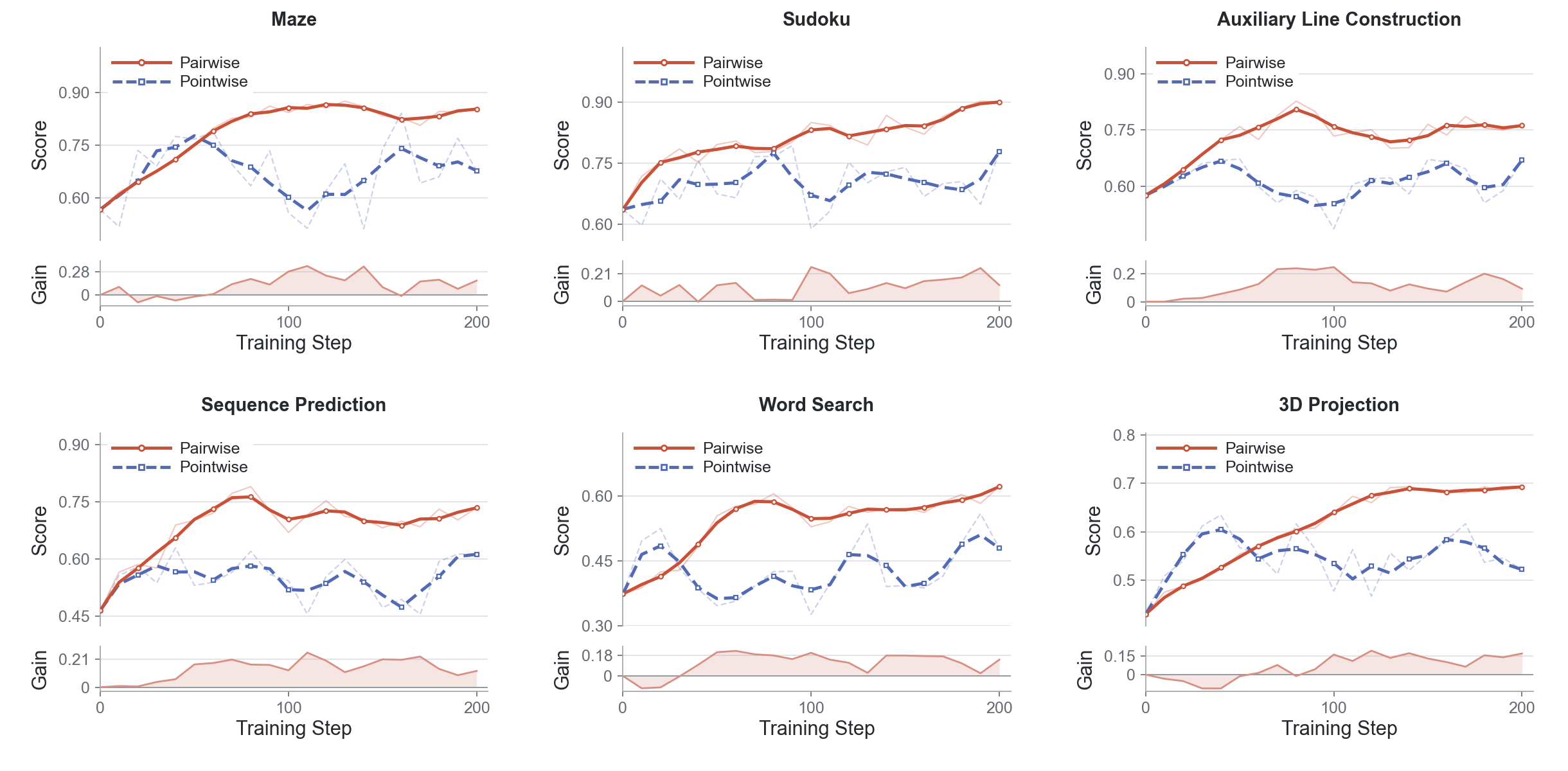}
    \caption{
    Training dynamics of pairwise and pointwise verification across
    six rule-governed tasks. Each panel shows validation scores
    and the score gap between the two methods.
    }
    \label{fig:six_task_training}
\end{figure}

\subsection{Sub-question-level Verification Accuracy}
\label{app:subquestion_accuracy}

The reliability of decomposed graded scoring depends on whether the verifier correctly judges individual sub-questions against the generated image. Even when the decomposition is appropriate, incorrect local judgments can introduce noise into the aggregated reward. We therefore directly evaluate sub-question-level verification accuracy on a held-out set of 200 binary sub-question samples. Three authors and two Ph.D. students each annotate 40 samples, collectively providing human labels for all 200 samples. We additionally obtain binary annotations from GPT-5.5 and Gemini-3.1-Pro on the same evaluation set. We compare the base model, the SFT-trained graded verifier, and the graded verifier further optimized with Decoupled RL, reporting accuracy against each annotation source separately.

\begin{table}[h]
\centering
\caption{
Sub-question-level verification accuracy (\%) on 200 held-out samples.
}
\label{tab:subquestion_accuracy}
\vspace{-2mm}
\begin{threeparttable}
\renewcommand{\arraystretch}{1.15}
\resizebox{0.7\linewidth}{!}{
\begin{tabular}{l|ccc}
\toprule
\multirow{2}{*}{\textbf{Methods}}
& \multicolumn{3}{c}{\textbf{Annotation Source}} \\
\cmidrule(lr){2-4}
& \textbf{GPT-5.5}
& \textbf{Gemini-3.1-Pro}
& \textbf{Human} \\
\midrule

Qwen3-VL-8B
& 68.5 & 69.5 & 67.0 \\

\ \ + SFT(Graded)
& 76.0 & 76.5 & 74.5 \\

\rowcolor{gray!10}
\ \ \textbf{+ SFT(Graded) + RL(Graded)}
& \textbf{94.5} & \textbf{96.0} & \textbf{94.0} \\

\bottomrule
\end{tabular}
}
\end{threeparttable}
\vspace{-3mm}
\end{table}

As shown in Table~\ref{tab:subquestion_accuracy}, Decoupled RL substantially improves fine-grained verification accuracy across all three annotation sources. These results support the effectiveness of Decoupled RL in improving sub-question-level binary judgments and provide empirical evidence for the reliability of the local decisions underlying decomposed graded scoring.

\subsection{Evaluation on Qwen-Image-2512}
\label{app:qwen_image_2512}

To examine whether the effectiveness of our Decomposed Graded Verifier extends to a larger text-to-image generation backbone, we further evaluate Qwen-Image-2512 (20B), complementing the experiments on Stable Diffusion 3.5 Medium (2.5B) and FLUX.1-dev (12B). We compare the original model with its reinforcement learning variants guided by the Discrete Binary Verifier and the Decomposed Graded Verifier, respectively, using GenEval and CLIPScore for evaluation.

\begin{table}[h]
\centering
\caption{
Performance of Qwen-Image-2512 on GenEval and CLIPScore
under different reward models.
}
\label{tab:qwen_image_2512}
\vspace{-2mm}
\begin{threeparttable}
\renewcommand{\arraystretch}{1.15}
\resizebox{0.99\linewidth}{!}{
\begin{tabular}{l|ccccccc|c}
\toprule
\multirow{2}{*}{\textbf{Methods}}
& \multicolumn{7}{c|}{\textbf{GenEval}}
& \multirow{2}{*}{\textbf{CLIPScore}} \\
\cmidrule(lr){2-8}
& Single Obj. & Two Obj. & Counting & Colors
& Position & Attri. Bind. & Overall & \\
\midrule

Qwen-Image-2512
& 0.99 & 0.90 & 0.77 & 0.86
& 0.38 & 0.52 & 0.74 & 34.1 \\

\ \ +RL(Discrete Binary Verifier)
& 1.00 & 0.92 & 0.79 & 0.87
& 0.42 & 0.55 & 0.76 & 34.4 \\

\rowcolor{gray!10}
\ \ \textbf{+RL(Decomposed Graded Verifier)}
& \textbf{1.00} & \textbf{0.96} & \textbf{0.85}
& \textbf{0.92} & \textbf{0.51} & \textbf{0.69}
& \textbf{0.82} & \textbf{35.3} \\

\bottomrule
\end{tabular}
}
\end{threeparttable}
\vspace{-3mm}
\end{table}

As shown in Table~\ref{tab:qwen_image_2512}, the Decomposed Graded Verifier improves the overall GenEval score from 0.74 to 0.82, outperforming the binary-reward counterpart by 0.06. The largest gains over the original model occur in attribute binding and position, which improve from 0.52 to 0.69 and from 0.38 to 0.51, respectively. These improvements suggest that fine-grained graded rewards are particularly useful for learning compositional requirements involving attribute associations and spatial relations. CLIPScore also increases from 34.1 to 35.3, compared with 34.4 under binary rewards. Together, these results extend the evidence for the effectiveness of decomposed graded rewards to a larger generation backbone with stronger baseline performance.

\subsection{Evaluation on InternVL3-8B}
\label{app:internvl_verifier}

To examine whether the benefits of our training strategy extend beyond the Qwen model family, we additionally adopt InternVL3-8B as the verifier backbone. We evaluate the base model, its graded SFT variant, and the graded verifier further optimized with Decoupled RL on ViVerBench.

\begin{table}[h]
\centering
\caption{
Performance on ViVerBench with InternVL3-8B as the verifier backbone.
We report seven selected sub-metrics with notable improvements after Decoupled RL.
Overall denotes the mean score across all 16 task categories, including those not displayed.
}
\label{tab:internvl_verifier}
\vspace{-2mm}
\begin{threeparttable}
\renewcommand{\arraystretch}{1.15}
\resizebox{0.99\linewidth}{!}{
\begin{tabular}{l|cc|cc|ccc|c}
\toprule
\multirow{2}{*}{\textbf{Methods}}
& \multicolumn{2}{c|}{\textbf{Concept Existence}}
& \multicolumn{2}{c|}{\textbf{Object Relation}}
& \multicolumn{3}{c|}{\textbf{Image Annotation}}
& \multirow{2}{*}{\shortstack{\textbf{Overall}\\\textbf{(16 tasks)}}} \\
\cmidrule(lr){2-3}
\cmidrule(lr){4-5}
\cmidrule(lr){6-8}
& Obj. & Attr. & Spat. & N-Spat.
& BBox & Point & Count & \\
\midrule

InternVL3-8B
& 0.513 & 0.578 & 0.667 & 0.619
& 0.750 & 0.619 & 0.604 & 0.594 \\

\ \ +SFT(Graded)
& 0.526 & 0.585 & 0.652 & 0.632
& 0.757 & 0.611 & 0.618 & 0.606 \\

\rowcolor{gray!10}
\ \ \textbf{+SFT(Graded)+RL(Graded)}
& \textbf{0.678} & \textbf{0.684}
& \textbf{0.707} & \textbf{0.724}
& \textbf{0.821} & \textbf{0.665}
& \textbf{0.649} & \textbf{0.652} \\

\bottomrule
\end{tabular}
}
\end{threeparttable}
\vspace{-3mm}
\end{table}

As shown in Table~\ref{tab:internvl_verifier}, Decoupled RL improves the mean score across all 16 task categories from 0.606 to 0.652 compared with graded SFT, with the seven displayed sub-metrics highlighting selected categories with notable gains. These results demonstrate that our training strategy is also effective on InternVL3-8B, extending its applicability beyond the Qwen family.

\subsection{Sensitivity to the Number of Sub-questions}
\label{app:decomposition_granularity}

To investigate how decomposition granularity affects global score estimation, we evaluate the trained Decomposed Graded Verifier on 200 held-out test samples. We explicitly instruct the verifier through the system prompt to generate sub-questions within three ranges: at most 10, 11--20, and 21--30. This experiment examines whether increasing the number of sub-questions consistently improves scoring accuracy or introduces additional errors through excessive decomposition.

\begin{table}[h]
\centering
\caption{
Effect of the number of sub-questions on global score estimation
on 200 held-out test samples. Lower mean alignment error is better.
}
\label{tab:decomposition_granularity}
\vspace{-2mm}
\begin{threeparttable}
\renewcommand{\arraystretch}{1.15}
\setlength{\tabcolsep}{12pt}
\begin{tabular}{l|c}
\toprule
\textbf{Number of Sub-questions}
& \textbf{Mean Alignment Error} $\downarrow$ \\
\midrule

$\leq 10$ & 0.083 \\

\rowcolor{gray!10}
11--20 & \textbf{0.039} \\

21--30 & 0.072 \\

\bottomrule
\end{tabular}
\end{threeparttable}
\vspace{-3mm}
\end{table}

As shown in Table~\ref{tab:decomposition_granularity}, the 11--20 setting achieves the lowest mean alignment error of 0.039, compared with 0.083 for at most 10 sub-questions and 0.072 for 21--30 sub-questions. These results suggest that an appropriate decomposition granularity improves scoring accuracy, while simply increasing the number of sub-questions does not guarantee better performance.

Excessive decomposition can introduce redundant questions or requirements unsupported by the original prompt. For example, given ``A red apple is placed on a table,'' the verifier may repeatedly ask whether an apple exists or introduce questions such as ``Is the apple a fruit?'' and ``Is the apple round?'' that do not correspond to explicitly requested visual constraints. Such questions can distort the aggregated score by repeatedly weighting the same requirement or introducing irrelevant criteria. We also observe that generating more sub-questions can lead to violations of the required output format, reducing the reliability of the final score. Together, these observations highlight the importance of balancing fine-grained coverage with faithful and non-redundant decomposition.

\section{Limitations and Future Works}\label{limitation}
Our world-modeling evaluation is restricted to game-based environments with explicit rules, where the ground-truth image is uniquely determined by the initial task state and rule instructions. We adopt these controlled tasks primarily to evaluate the effectiveness of our pairwise verification paradigm. Although our results demonstrate the advantages of pairwise verification in these settings, these tasks capture only part of the complexity and diversity of real-world environments.

In future work, we plan to extend our investigation to open-ended settings with non-unique ground truths, construct larger and more realistic world-modeling datasets, and scale both the verifier and the generator to stronger backbones. Our long-term goal is to develop multimodal generative models that can reason about, simulate, and interact with complex real-world scenes.

\section{Decomposition Example}
\begin{tcolorbox}[title = {Text-to-Image (1)}, breakable, fontupper=\small, fonttitle=\small]

\begin{minipage}[c]{0.4\textwidth}
\vspace{0pt}
\includegraphics[width=\linewidth]{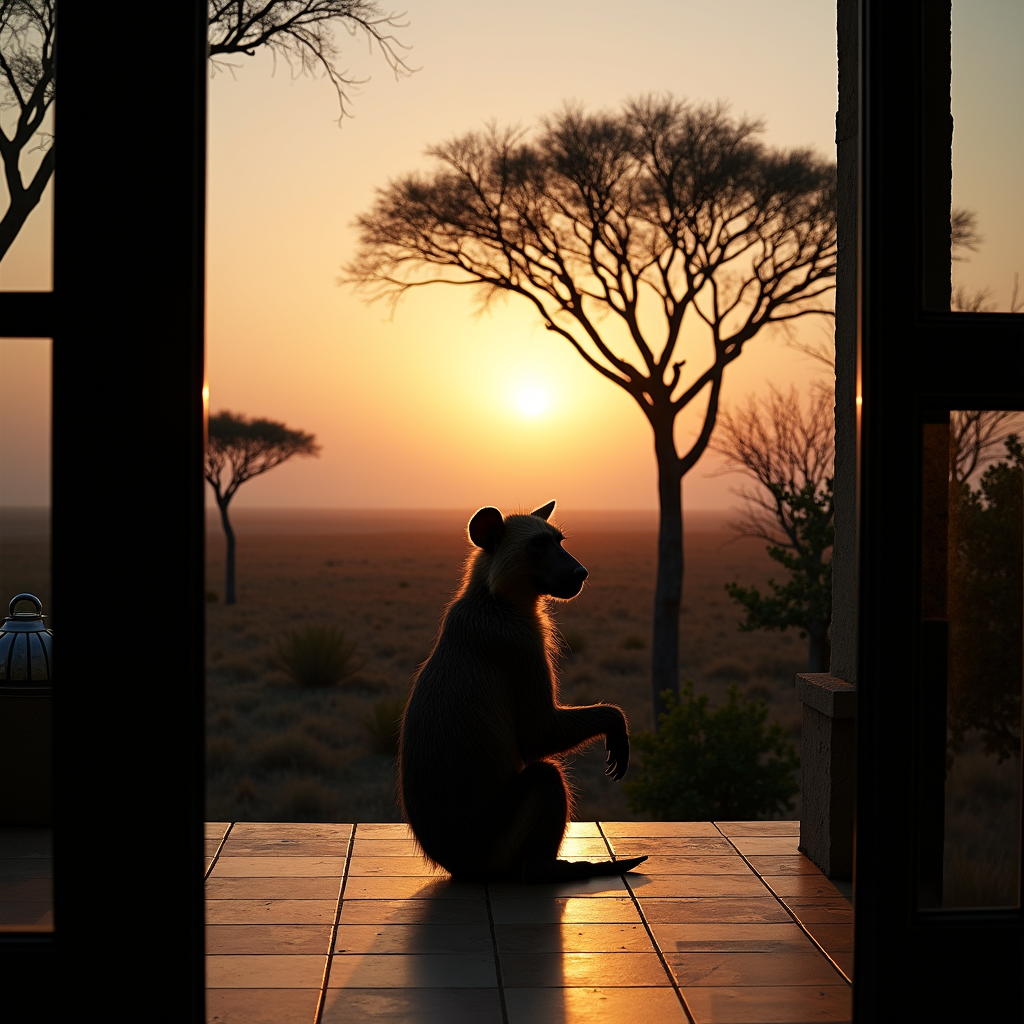}
\end{minipage}
\hfill
\begin{minipage}[c]{0.58\textwidth}
\vspace{0pt}
\textbf{Prompt}:

\vspace{0.4em}

A baboon is silhouetted on a patio in a top-down, monochromatic view; acacia trees stark against the golden hour dusk.

\end{minipage}

\vspace{0.8em}

\noindent\textbf{Decomposition Result:}

\vspace{0.3em}

\begin{itemize}
    \item Is there a baboon present in the image? $\rightarrow$ Yes, Score: 1
    \item Is the baboon depicted in silhouette? $\rightarrow$ Yes, Score: 1
    \item Is the baboon situated on a patio or tiled surface? $\rightarrow$ Yes, Score: 1
    \item Is the camera perspective a top-down view? $\rightarrow$ No, Score: 0
    \item Is the image monochromatic? $\rightarrow$ No, Score: 0
    \item Are there acacia trees visible in the background? $\rightarrow$ Yes, Score: 1
    \item Is the scene set during golden hour dusk? $\rightarrow$ Yes, Score: 1
\end{itemize}

\vspace{0.3em}

\noindent\textbf{Final Alignment Score:}

\vspace{0.3em}

$$\frac{5}{7}=0.71$$

\end{tcolorbox}

\begin{tcolorbox}[title = {Text-to-Image (2)}, breakable, fontupper=\small, fonttitle=\small]

\begin{minipage}[c]{0.4\textwidth}
\vspace{0pt}
\includegraphics[width=\linewidth]{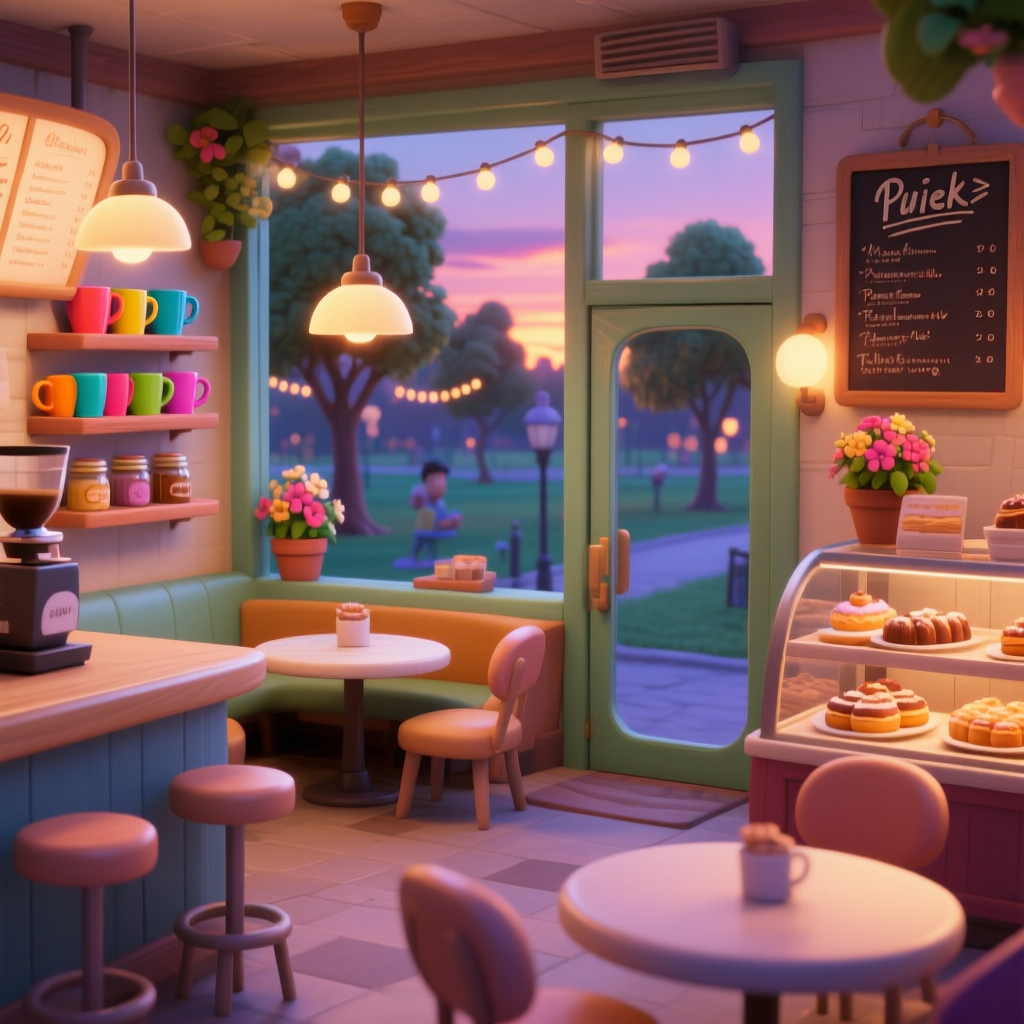}
\end{minipage}
\hfill
\begin{minipage}[c]{0.58\textwidth}
\vspace{0pt}
\textbf{Prompt}:

\vspace{0.4em}

Isometric 3D café interior, Pixar-style rendering, whimsical aesthetic. Features coffee counter, colorful shelves, rounded furniture, pastry case, and hanging lamps. Large glass door opens to a park with string lights at sunset. Warm glowing lighting, soft curved edges, cozy evening atmosphere.

\end{minipage}

\vspace{0.8em}

\noindent\textbf{Decomposition Result:}

\vspace{0.3em}

\begin{itemize}
    \item Is the image a 3D rendering of a café interior with a Pixar-style aesthetic? $\rightarrow$ Yes, Score: 1
    \item Is the viewpoint isometric? $\rightarrow$ No, Score: 0
    \item Does the interior feature a coffee counter? $\rightarrow$ Yes, Score: 1
    \item Are there shelves with colorful items? $\rightarrow$ Yes, Score: 1
    \item Are the furniture and objects rounded with soft curved edges? $\rightarrow$ Yes, Score: 1
    \item Does the café include both hanging lamps and a pastry case? $\rightarrow$ Yes, Score: 1
    \item Does a large glass door open to an outdoor park scene? $\rightarrow$ Yes, Score: 1
    \item Are string lights visible in the outdoor park scene? $\rightarrow$ Yes, Score: 1
    \item Is the scene set at sunset with warm, glowing, cozy lighting? $\rightarrow$ Yes, Score: 1
\end{itemize}

\vspace{0.3em}

\noindent\textbf{Final Alignment Score:}

\vspace{0.3em}

$$\frac{8}{9}=0.89$$

\end{tcolorbox}

\begin{tcolorbox}[title = {Text-to-Image (3)}, breakable, fontupper=\small, fonttitle=\small]

\begin{minipage}[c]{0.4\textwidth}
\vspace{0pt}
\includegraphics[width=\linewidth]{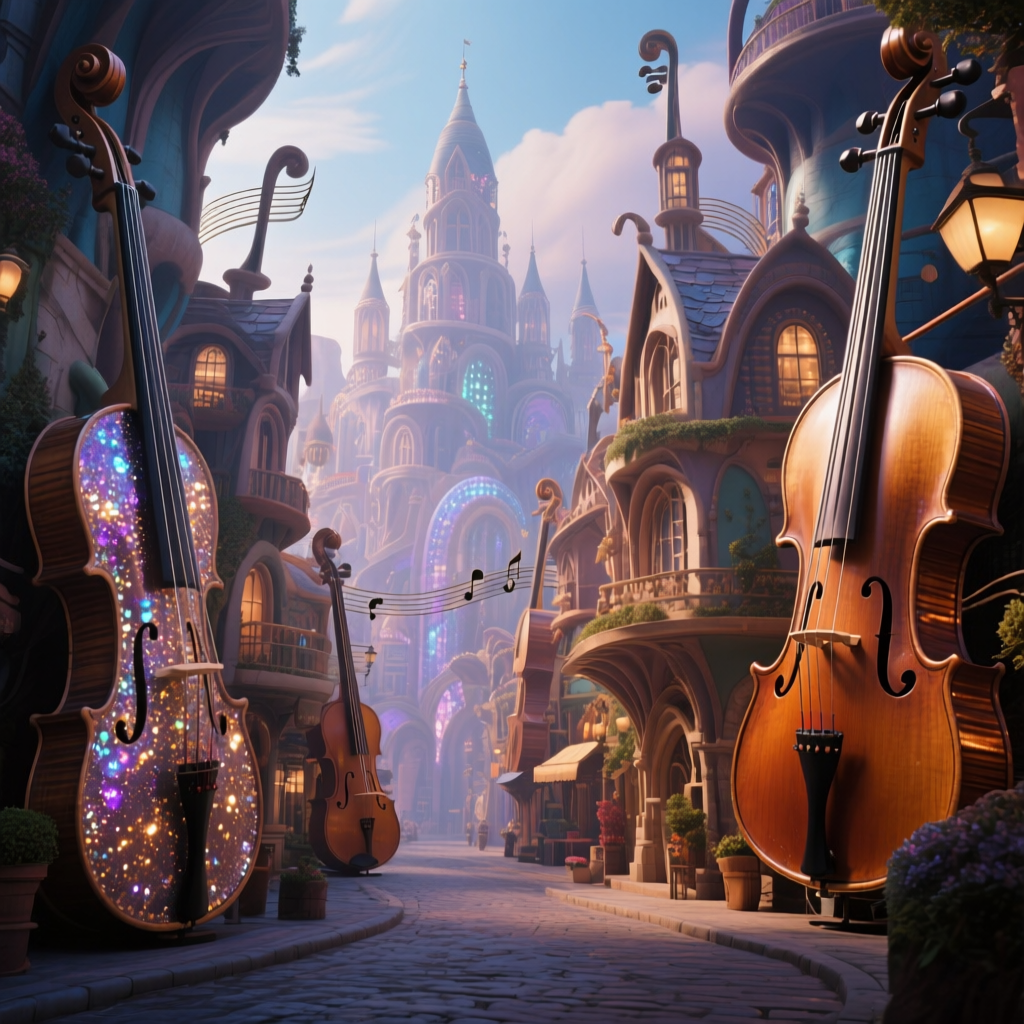}
\end{minipage}
\hfill
\begin{minipage}[c]{0.58\textwidth}
\vspace{0pt}
\textbf{Prompt}:

\vspace{0.4em}

Imagine an enchanted city where the streets resonate like symphonies, and the architecture is formed from orchestral instruments and glimmering harmonies.

\end{minipage}

\vspace{0.8em}

\noindent\textbf{Decomposition Result:}

\vspace{0.3em}

\begin{itemize}
    \item Does the image show an enchanted city? $\rightarrow$ Yes, Score: 1
    \item Is there a clear city street scene? $\rightarrow$ Yes, Score: 1
    \item Are musical notes or symphony elements visible? $\rightarrow$ Yes, Score: 1
    \item Are orchestral instruments visible? $\rightarrow$ Yes, Score: 1
    \item Are the instruments integrated into the cityscape? $\rightarrow$ Yes, Score: 1
    \item Is the architecture formed from instruments? $\rightarrow$ No, Score: 0
    \item Are glimmering lights or particles visible? $\rightarrow$ Yes, Score: 1
    \item Are the glimmering elements linked to the instruments? $\rightarrow$ Yes, Score: 1
    \item Does the scene combine city, music, and magic coherently? $\rightarrow$ Yes, Score: 1
\end{itemize}

\vspace{0.3em}

\noindent\textbf{Final Alignment Score:}

\vspace{0.3em}

$$\frac{8}{9}=0.89$$

\end{tcolorbox}

\begin{tcolorbox}[title = {Text-to-Image (4)}, breakable, fontupper=\small, fonttitle=\small]

\begin{minipage}[c]{0.4\textwidth}
\vspace{0pt}
\includegraphics[width=\linewidth]{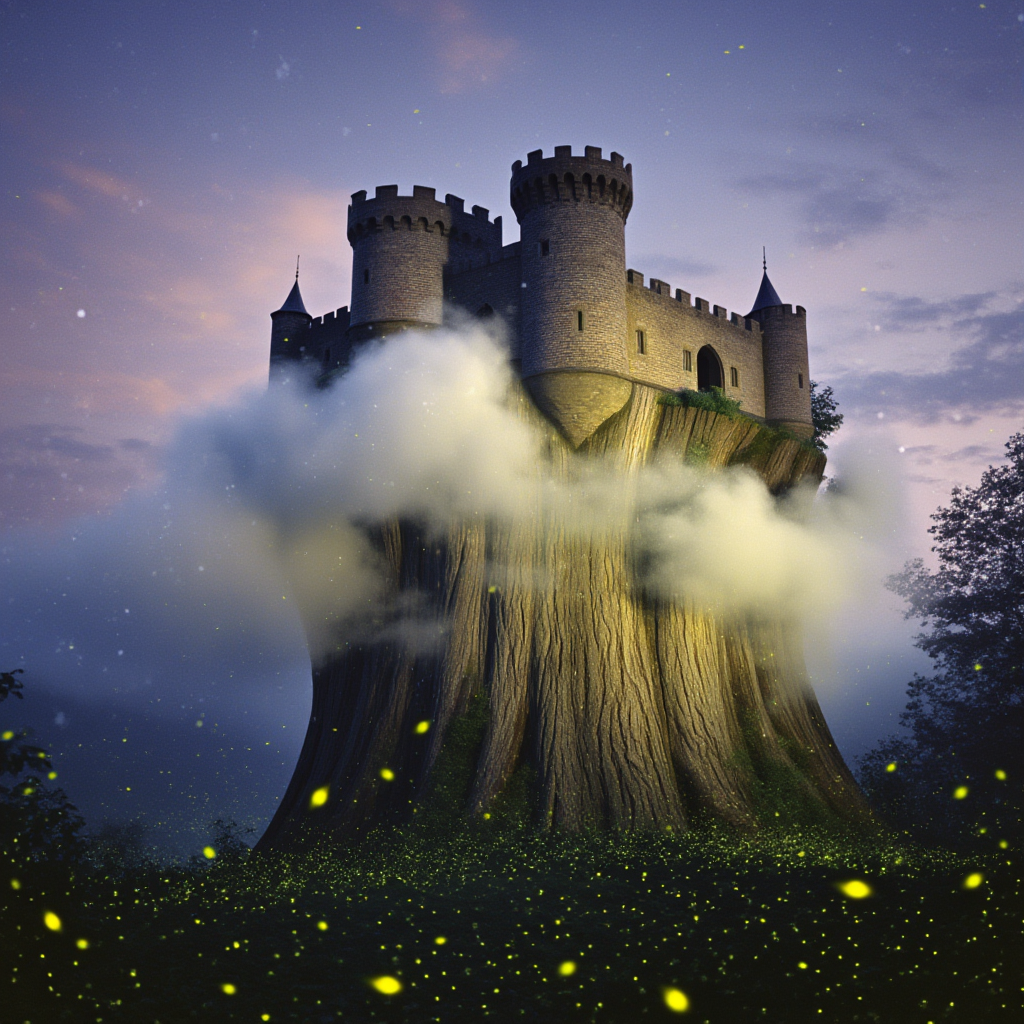}
\end{minipage}
\hfill
\begin{minipage}[c]{0.58\textwidth}
\vspace{0pt}
\textbf{Prompt}:

\vspace{0.4em}

A retro-themed photograph of a medieval castle resting atop a colossal tree trunk, enveloped by clouds, shimmering fireflies, and a dusky twilight sky. The scene is illuminated by fireflies.

\end{minipage}

\vspace{0.8em}

\noindent\textbf{Decomposition Result:}

\vspace{0.3em}

\begin{itemize}
    \item Does the image depict a medieval castle? $\rightarrow$ Yes, Score: 1
    \item Is the castle resting atop a colossal tree trunk? $\rightarrow$ Yes, Score: 1
    \item Is the scene enveloped by clouds? $\rightarrow$ Yes, Score: 1
    \item Are shimmering fireflies visible in the image? $\rightarrow$ Yes, Score: 1
    \item Is the sky a dusky twilight sky? $\rightarrow$ Yes, Score: 1
    \item Is the image a retro-themed photograph? $\rightarrow$ No, Score: 0
    \item Is the scene illuminated by the fireflies? $\rightarrow$ No, Score: 0
\end{itemize}

\vspace{0.3em}

\noindent\textbf{Final Alignment Score:}

\vspace{0.3em}

$$\frac{5}{7}=0.71$$

\end{tcolorbox}

\begin{tcolorbox}[title = {Text-to-Image (5)}, breakable, fontupper=\small, fonttitle=\small]

\begin{minipage}[c]{0.4\textwidth}
\vspace{0pt}
\includegraphics[width=\linewidth]{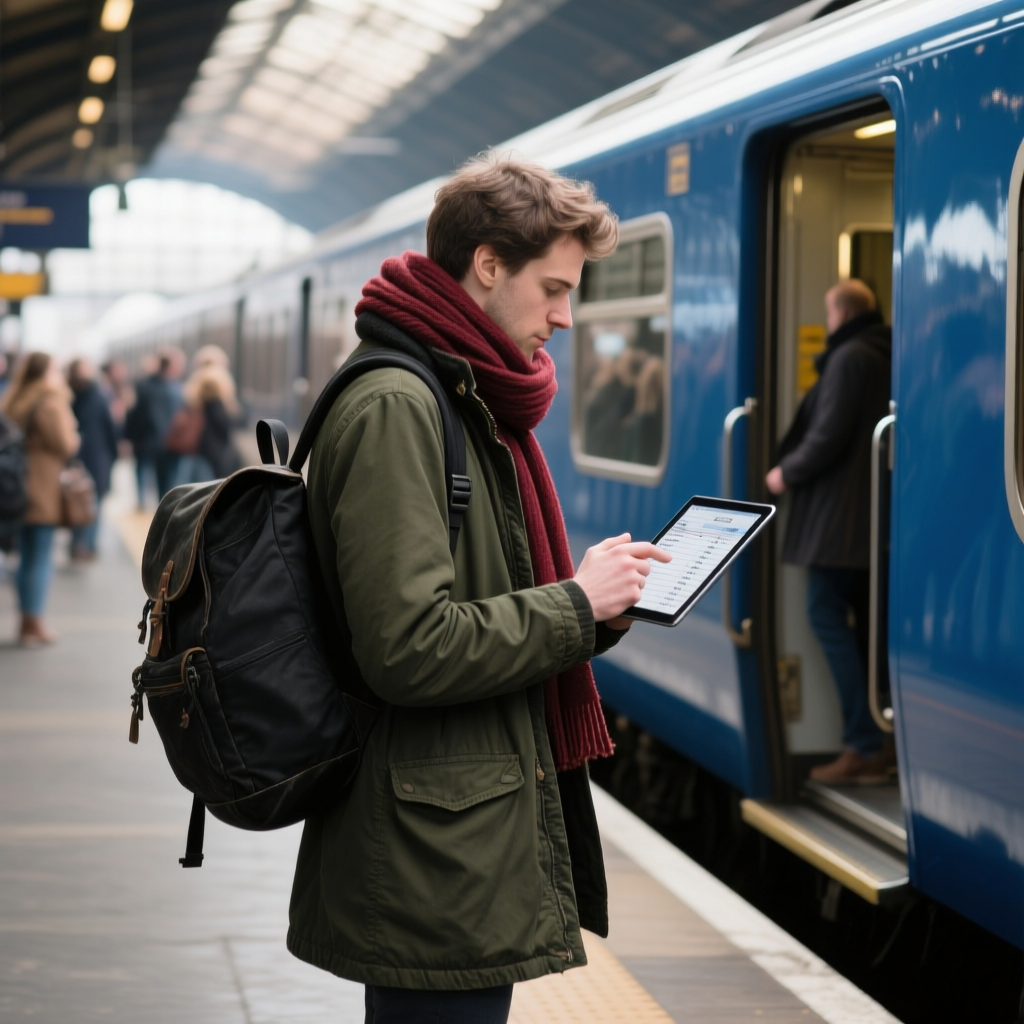}
\end{minipage}
\hfill
\begin{minipage}[c]{0.58\textwidth}
\vspace{0pt}
\textbf{Prompt}:

\vspace{0.4em}

In a bustling train station, a traveler in a dark olive jacket and maroon scarf wears a well-worn black backpack. They focus on a tablet, likely checking schedules. Nearby, a polished blue train waits at the platform, ready to welcome passengers.

\end{minipage}

\vspace{0.8em}

\noindent\textbf{Decomposition Result:}

\vspace{0.3em}

\begin{itemize}
    \item Is the scene set in a bustling train station? $\rightarrow$ Yes, Score: 1
    \item Is the traveler wearing a dark olive jacket? $\rightarrow$ Yes, Score: 1
    \item Is the traveler wearing a maroon scarf? $\rightarrow$ Yes, Score: 1
    \item Is the traveler wearing a well-worn black backpack? $\rightarrow$ No, Score: 0
    \item Is the traveler focusing on a tablet? $\rightarrow$ Yes, Score: 1
    \item Is there a polished blue train waiting at the platform? $\rightarrow$ Yes, Score: 1
    \item Is the train's door open to welcome passengers? $\rightarrow$ Yes, Score: 1
\end{itemize}

\vspace{0.3em}

\noindent\textbf{Final Alignment Score:}

\vspace{0.3em}

$$\frac{6}{7}=0.86$$

\end{tcolorbox}

\begin{tcolorbox}[title = {Rule-Governed World Worlding Pairwise}, breakable, fontupper=\small, fonttitle=\small]

\begin{minipage}[c]{0.4\textwidth}
\vspace{0pt}
\centering
\includegraphics[width=\linewidth]{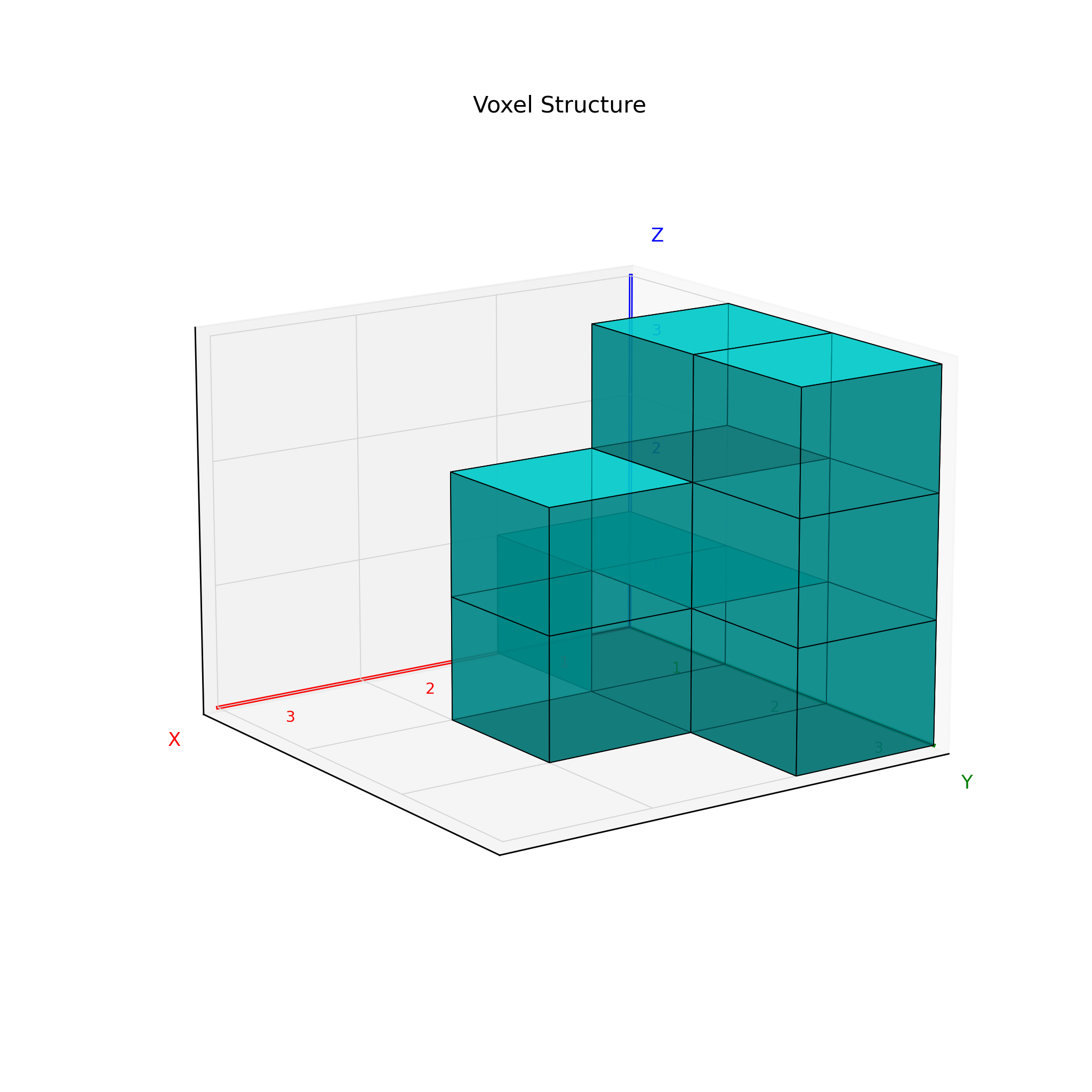}

\vspace{0.3em}
{\scriptsize Initial Image}
\end{minipage}
\hfill
\begin{minipage}[c]{0.58\textwidth}
\vspace{0pt}
\textbf{Game Rule}:

\vspace{0.4em}

The objective is to generate a 2D orthographic front-view projection of the given 3D voxel structure. 
The front view is obtained by looking along the X-axis and projecting the structure onto the Y-Z plane.

\end{minipage}

\vspace{0.8em}

\begin{minipage}[t]{0.4\textwidth}
\centering
\includegraphics[width=\linewidth]{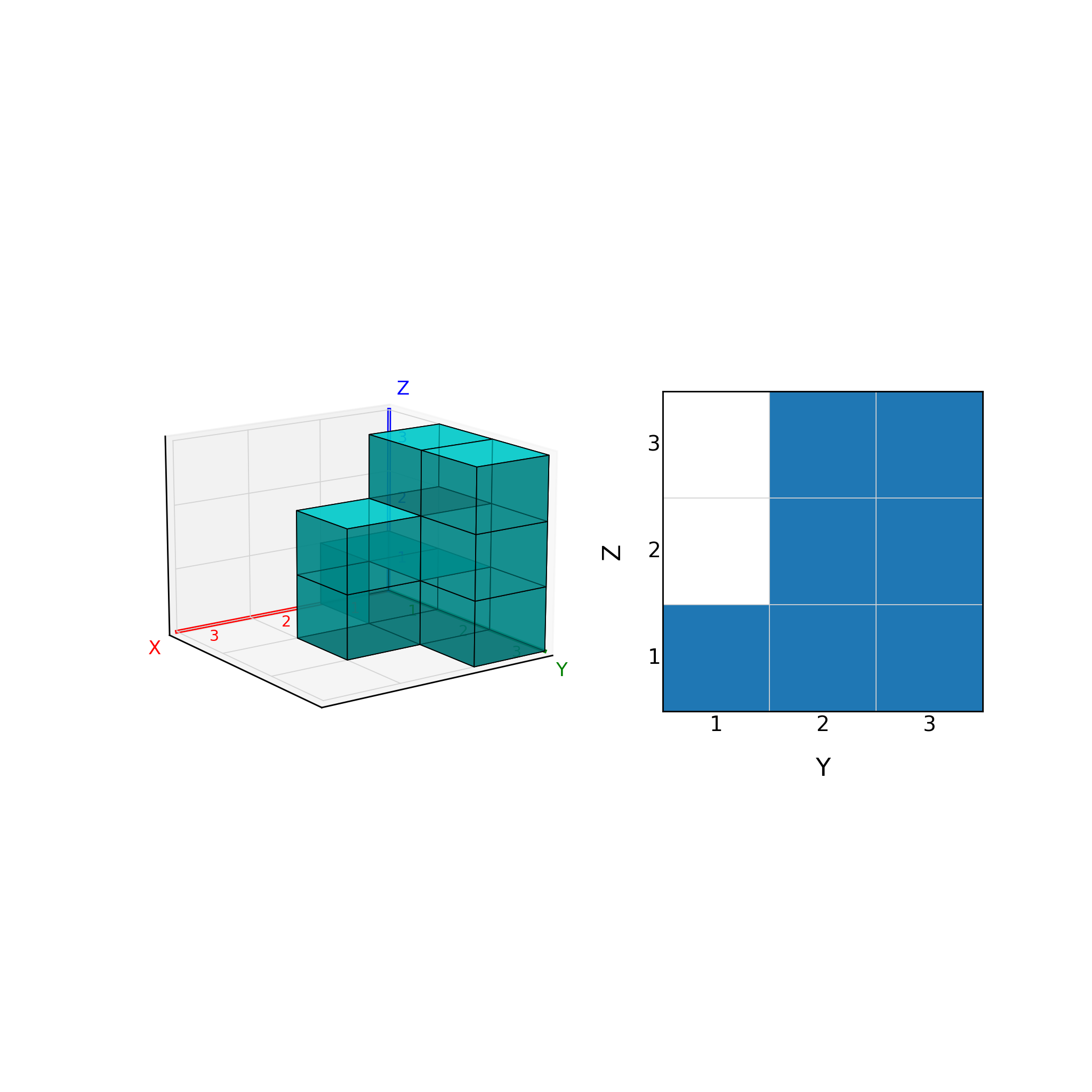}

\vspace{0.3em}
{\scriptsize Reference Image}
\end{minipage}
\hfill
\begin{minipage}[t]{0.4\textwidth}
\centering
\includegraphics[width=\linewidth]{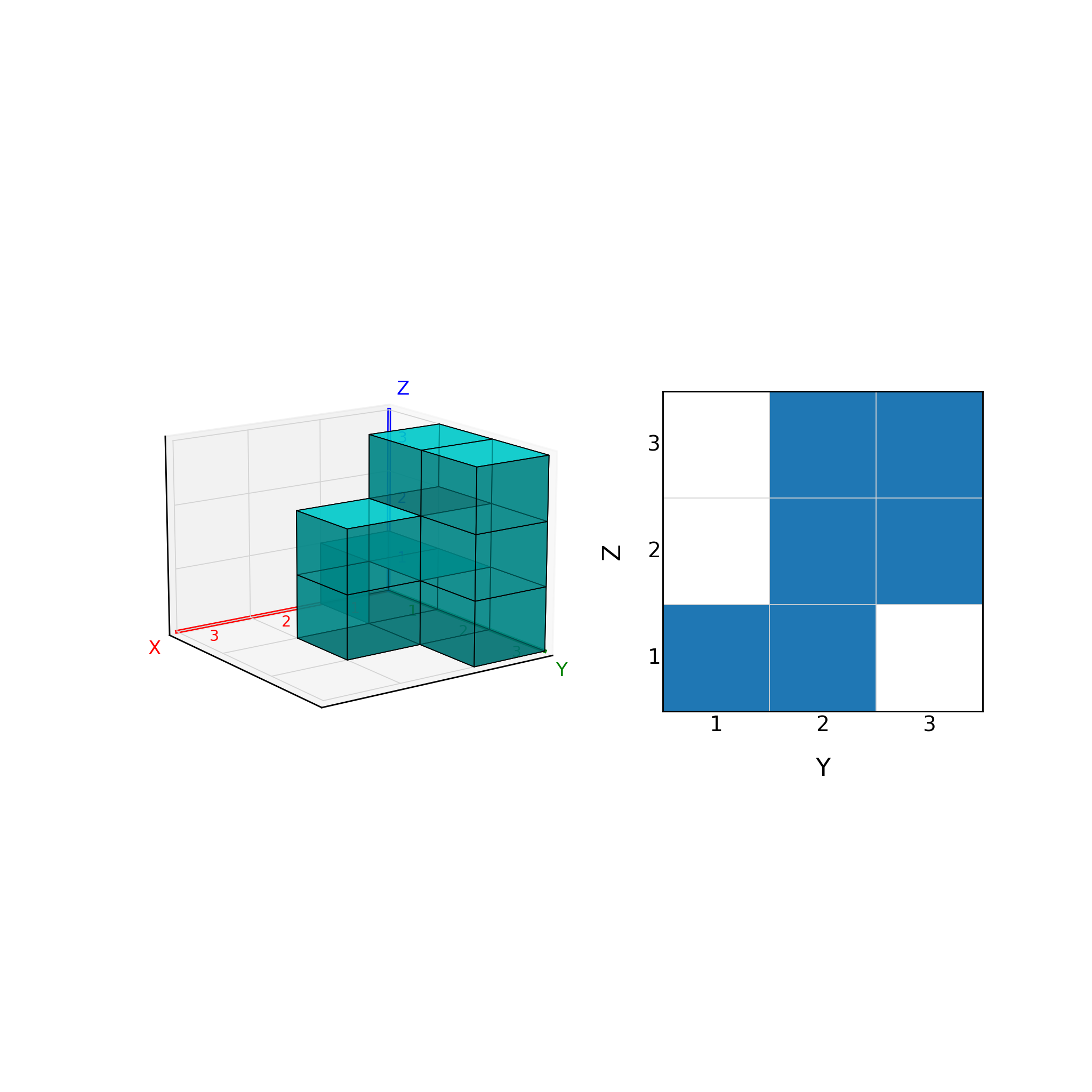}

\vspace{0.3em}
{\scriptsize Generated Image}
\end{minipage}

\vspace{0.8em}

\textbf{Decomposition Reasoning}:

\vspace{0.4em}

So, let's tackle this step by step.

\vspace{0.4em}

The task is to evaluate whether the generated image correctly produces the 2D orthographic front-view projection of the given 3D voxel structure. 
Since the front view is obtained by looking along the X-axis, the projection should be represented on the Y-Z plane.

\vspace{0.4em}

I use the Reference Image as the correct target and compare the Generated Image with it cell by cell. 
Each position in the $3 \times 3$ Y-Z grid can be treated as a small sub-question: whether the generated cell matches the corresponding reference cell.

\vspace{0.4em}

The Reference Image shows the projection as follows:

\vspace{0.3em}

\begin{itemize}
    \item $Z=3$: empty, filled, filled
    \item $Z=2$: empty, filled, filled
    \item $Z=1$: filled, filled, filled
\end{itemize}

\vspace{0.3em}

Based on the Reference Image, I decompose the comparison into the following cell-level checks:

\vspace{0.3em}

\begin{itemize}
    \item Is the cell at $(Y=1,Z=1)$ correct? The Reference Image fills this cell, and the Generated Image also fills it. $\rightarrow$ Yes, Score: 1
    \item Is the cell at $(Y=2,Z=1)$ correct? The Reference Image fills this cell, and the Generated Image also fills it. $\rightarrow$ Yes, Score: 1
    \item Is the cell at $(Y=3,Z=1)$ correct? The Reference Image fills this cell, but the Generated Image leaves it empty. $\rightarrow$ No, Score: 0
    \item Is the cell at $(Y=1,Z=2)$ correct? The Reference Image leaves this cell empty, and the Generated Image also leaves it empty. $\rightarrow$ Yes, Score: 1
    \item Is the cell at $(Y=2,Z=2)$ correct? The Reference Image fills this cell, and the Generated Image also fills it. $\rightarrow$ Yes, Score: 1
    \item Is the cell at $(Y=3,Z=2)$ correct? The Reference Image fills this cell, and the Generated Image also fills it. $\rightarrow$ Yes, Score: 1
    \item Is the cell at $(Y=1,Z=3)$ correct? The Reference Image leaves this cell empty, and the Generated Image also leaves it empty. $\rightarrow$ Yes, Score: 1
    \item Is the cell at $(Y=2,Z=3)$ correct? The Reference Image fills this cell, and the Generated Image also fills it. $\rightarrow$ Yes, Score: 1
    \item Is the cell at $(Y=3,Z=3)$ correct? The Reference Image fills this cell, and the Generated Image also fills it. $\rightarrow$ Yes, Score: 1
\end{itemize}

\vspace{0.4em}

From these sub-question judgements, the Generated Image matches the Reference Image in 8 out of 9 cells. 
The only mismatch is at $(Y=3,Z=1)$, where the Reference Image shows an empty cell but the Generated Image marks it as filled.

\vspace{0.4em}

Thus, there are 9 sub-questions in total, with 8 correct and 1 incorrect.

\[
\text{Score}=\frac{8}{9}=0.89.
\]

\end{tcolorbox}

\begin{tcolorbox}[title = {Rule-Governed World Worlding Pointwise}, breakable, fontupper=\small, fonttitle=\small]

\begin{minipage}[c]{0.4\textwidth}
\vspace{0pt}
\centering
\includegraphics[width=\linewidth]{images/initial.png}

\vspace{0.3em}
{\scriptsize Initial Image}
\end{minipage}
\hfill
\begin{minipage}[c]{0.4\textwidth}
\vspace{0pt}
\centering
\includegraphics[width=\linewidth]{images/generated.png}

\vspace{0.3em}
{\scriptsize Generated Image}
\end{minipage}

\vspace{0.8em}

\textbf{Decomposition Reasoning}:

\vspace{0.4em}

So, let's first work out the target projection implied by the initial 3D structure.

\vspace{0.4em}

The required view is produced by observing the voxel structure along the X-axis. 
In this process, the X dimension is removed, and the remaining projection is represented on the Y-Z plane, with Y as the horizontal axis and Z as the vertical axis.

\vspace{0.4em}

Thus, for every $(Y,Z)$ coordinate, I need to decide whether the projection cell should be occupied. 
If at least one voxel exists somewhere along the X direction at that same $(Y,Z)$ coordinate, the cell should be filled; otherwise, it should be empty.

\vspace{0.4em}

By inspecting the initial 3D structure, the expected occupancy pattern is:

\vspace{0.3em}

\begin{itemize}
    \item For $Y=1$, the structure occupies $Z=1$ and $Z=2$.
    \item For $Y=2$, the structure occupies $Z=1$, $Z=2$, and $Z=3$.
    \item For $Y=3$, the structure occupies $Z=1$, $Z=2$, and $Z=3$.
\end{itemize}

\vspace{0.3em}

So the inferred target projection should mark the following cells as filled:
\[
\begin{aligned}
&(Y=1,Z=1),\ (Y=1,Z=2),\\
&(Y=2,Z=1),\ (Y=2,Z=2),\ (Y=2,Z=3),\\
&(Y=3,Z=1),\ (Y=3,Z=2),\ (Y=3,Z=3).
\end{aligned}
\]

Next, I break the projection into individual cell-level questions and compare the generated result with this inferred target.

\vspace{0.3em}

\begin{itemize}
    \item Is $(Y=1,Z=1)$ correctly represented? The inferred target requires this cell to be filled, and the generated projection also fills it. $\rightarrow$ Yes, Score: 1

    \item Is $(Y=1,Z=2)$ correctly represented? The inferred target requires this cell to be filled, but the generated projection leaves it empty. $\rightarrow$ No, Score: 0

    \item Is $(Y=1,Z=3)$ correctly represented? The inferred target requires this cell to be empty, and the generated projection also leaves it empty. $\rightarrow$ Yes, Score: 1

    \item Is $(Y=2,Z=1)$ correctly represented? The inferred target requires this cell to be filled, and the generated projection also fills it. $\rightarrow$ Yes, Score: 1

    \item Is $(Y=2,Z=2)$ correctly represented? The inferred target requires this cell to be filled, and the generated projection also fills it. $\rightarrow$ Yes, Score: 1

    \item Is $(Y=2,Z=3)$ correctly represented? The inferred target requires this cell to be filled, and the generated projection also fills it. $\rightarrow$ Yes, Score: 1

    \item Is $(Y=3,Z=1)$ correctly represented? The inferred target requires this cell to be filled, but the generated projection leaves it empty. $\rightarrow$ No, Score: 0

    \item Is $(Y=3,Z=2)$ correctly represented? The inferred target requires this cell to be filled, and the generated projection also fills it. $\rightarrow$ Yes, Score: 1

    \item Is $(Y=3,Z=3)$ correctly represented? The inferred target requires this cell to be filled, and the generated projection also fills it. $\rightarrow$ Yes, Score: 1
\end{itemize}

\vspace{0.4em}

After aggregating these judgements, the generated projection correctly represents most of the inferred Y-Z occupancy pattern. 
The two failed sub-questions correspond to $(Y=1,Z=2)$ and $(Y=3,Z=1)$, where the generated projection leaves the cells empty even though they should be filled.

\vspace{0.4em}

Therefore, among the 9 decomposed sub-questions, 7 are judged as correct and 2 are judged as incorrect.

\[
\text{Score}=\frac{7}{9}=0.78.
\]

\end{tcolorbox}

\begin{tcolorbox}[
    title = {Text-to-Image Verifier Prompt},
    breakable,
    colback = green!3,
    colframe = green!45!black,
    coltitle = white,
    colbacktitle = green!45!black,
    fontupper = \small,
    fonttitle = \bfseries\small,
    arc = 1.5mm,
    boxrule = 0.7pt
]

You are a powerful and trustworthy \textbf{Text-to-Image Alignment Evaluator}.

\vspace{0.5em}

Your responsibility is to evaluate the semantic alignment between a generated image and a text prompt. You must decompose the text prompt into verifiable atomic questions and score the consistency between the image and the prompt based strictly on visual evidence.

\vspace{0.5em}

\textbf{Inputs}

\vspace{0.5em}

You will be given the following inputs:

\begin{itemize}
    \item \textbf{Text Prompt}: the textual description used to generate the image.
    \item \textbf{Generated Image}: the image produced by the text-to-image model.
\end{itemize}

\vspace{0.5em}

\textbf{Evaluation Workflow}

\vspace{0.5em}

\textbf{Step 1: Prompt Decomposition}

\vspace{0.4em}

Decompose the text prompt into a set of atomic questions that can be directly verified from the image. Each question must originate from facts explicitly stated in the text prompt and test one clear semantic assertion.

\vspace{0.5em}

\textbf{Step 2: Image-Based Answering}

\vspace{0.4em}

Answer each question solely based on the visual information present in the Generated Image. Do not introduce assumptions beyond the text prompt. Do not infer visual information that is absent, ambiguous, or cannot be confirmed from the image.

\vspace{0.5em}

\textbf{Step 3: Consistency Judgement and Scoring}

\vspace{0.4em}

For each question, determine whether the image-based answer is consistent with the information explicitly specified in the text prompt. If the image-based answer is consistent with the prompt, the question receives 1 point. Otherwise, it receives 0 points.

\vspace{0.5em}

\textbf{Question Decomposition Principles}

\vspace{0.5em}

The decomposed questions must satisfy the following requirements:

\begin{itemize}
    \item Each question must be binary, with a Yes or No answer.
    \item Each question must be objectively verifiable through visual evidence.
    \item Questions involving subjective judgment, emotions, aesthetics, or other non-verifiable properties are forbidden.
    \item The number of questions should be naturally determined by the semantic complexity of the text prompt.
\end{itemize}

The decomposition should follow the semantic order below when applicable:

\begin{itemize}
    \item \textbf{Objects}: existence, quantity, and category.
    \item \textbf{Object Attributes}: color, shape, material, appearance, actions, and visible states.
    \item \textbf{Spatial Relations}: relative position, orientation, and layout.
    \item \textbf{Non-Spatial Relations}: interactions and action relations between entities.
\end{itemize}

\vspace{0.5em}

\textbf{Output Requirements}

\vspace{0.5em}

Output only a valid JSON object containing exactly two top-level fields: "sub\_questions" and "final\_score". Do not include Markdown code fences or additional text.

\par\vspace{0.7em}
\noindent
\begingroup
\renewcommand{\arraystretch}{1.08}
\begin{tabular}{@{}l@{}}
\{ \\
\hspace*{1.5em}"sub\_questions": [ \\
\hspace*{3em}\{ \\
\hspace*{4.5em}"id": 1, \\
\hspace*{4.5em}"question": "A decomposed sub-question", \\
\hspace*{4.5em}"answer": "True or False", \\
\hspace*{4.5em}"explanation": "Brief error explanation if False; otherwise an empty string" \\
\hspace*{3em}\} \\
\hspace*{1.5em}], \\
\hspace*{1.5em}"final\_score": x \\
\}
\end{tabular}
\endgroup
\par\vspace{0.7em}

Replace the placeholders above with actual evaluation results. Include one entry for each distinct decomposed sub-question, with consecutive IDs starting at 1.

\vspace{0.5em}

Here, $\mathrm{x} = \mathrm{C}/\mathrm{N}$, where $\mathrm{N}$ is the total number of entries and $\mathrm{C}$ is the number answered "True". Output $\mathrm{x}$ as a JSON number rounded to two decimal places.

\vspace{0.5em}

\textbf{Inputs for This Evaluation}

\vspace{0.5em}

\textbf{Text Prompt:} \$ \{ prompt \}

\vspace{0.5em}

The attached image is the Generated Image. \textless image\textgreater

\end{tcolorbox}

\begin{tcolorbox}[
    title = {Rule-Governed World Modeling Verifier Prompt-Pairwise},
    breakable,
    colback = green!3,
    colframe = green!45!black,
    coltitle = white,
    colbacktitle = green!45!black,
    fontupper = \small,
    fonttitle = \bfseries\small,
    arc = 1.5mm,
    boxrule = 0.7pt
]

You are a powerful and trustworthy \textbf{Image Editing Evaluator}.

\vspace{0.5em}

Your responsibility is to evaluate whether an image editing model successfully completes a game-based image editing task according to the provided rules.

\vspace{0.5em}

\textbf{Inputs}

\vspace{0.5em}

You will be given the following inputs:

\begin{itemize}
    \item \textbf{Initial Image}: the original game state given to the image editing model.
    \item \textbf{Generated Image}: the output produced by the image editing model.
    \item \textbf{Reference Image}: the ground-truth result that correctly follows all rules.
    \item \textbf{Game Rules}: the rules governing the game and editing constraints.
\end{itemize}

\vspace{0.5em}

\textbf{Evaluation Workflow}

\vspace{0.5em}

\textbf{Step 1: Rule Decomposition}

\vspace{0.4em}

Independently decompose the Game Rules into atomic, visually checkable rules, even if they are already decomposed. The resulting rules must remain consistent with the provided Game Rules.

\vspace{0.5em}

\textbf{Step 2: Rule Violation Checking}

\vspace{0.4em}

Use the Reference Image as the correct standard. \textbf{Compare the Generated Image with the Reference Image} rule by rule, and determine whether each rule is violated.

\vspace{0.5em}

\textbf{Step 3: Scoring}

\vspace{0.4em}

Let the total number of decomposed atomic rules be $\mathrm{N}$, and let the number of violated rules be $\mathrm{V}$. The score of the Generated Image is:

\[
\textnormal{Score} = \frac{\mathrm{N} - \mathrm{V}}{\mathrm{N}}.
\]

The final score should reflect how consistently the Generated Image follows the rules compared with the Reference Image.

\vspace{0.5em}

\textbf{Output Requirements}

\vspace{0.5em}

Output only a valid JSON object with exactly two top-level fields: "sub\_questions" and "final\_score". Do not include Markdown code fences or additional text.

\par\vspace{0.7em}
\noindent
\begingroup
\renewcommand{\arraystretch}{1.08}
\begin{tabular}{@{}l@{}}
\{ \\
\hspace*{1.5em}"sub\_questions": [ \\
\hspace*{3em}\{ \\
\hspace*{4.5em}"id": 1, \\
\hspace*{4.5em}"question": "A decomposed rule sub-question", \\
\hspace*{4.5em}"answer": "True or False", \\
\hspace*{4.5em}"explanation": "Brief rule-violation explanation if False; otherwise an empty string" \\
\hspace*{3em}\} \\
\hspace*{1.5em}], \\
\hspace*{1.5em}"final\_score": x \\
\}
\end{tabular}
\endgroup
\par\vspace{0.7em}

Replace the placeholders above with actual evaluation results. Include one entry for each distinct decomposed atomic rule, with consecutive IDs starting at 1.

\vspace{0.5em}

Here, $\mathrm{x}$ is the score calculated in Step 3, expressed as a JSON number rounded to two decimal places.

\vspace{0.5em}

\textbf{Inputs for This Evaluation}

\vspace{0.5em}

\textbf{Game Rules:} \$ \{ game\_rules \}

\vspace{0.5em}

Images are attached in this order: Initial Image, Generated Image, Reference Image. \textless image\textgreater \textless image\textgreater \textless image\textgreater

\end{tcolorbox}

\begin{tcolorbox}[
    title = {Rule-Governed World Modeling Verifier Prompt-Pointwise},
    breakable,
    colback = green!3,
    colframe = green!45!black,
    coltitle = white,
    colbacktitle = green!45!black,
    fontupper = \small,
    fonttitle = \bfseries\small,
    arc = 1.5mm,
    boxrule = 0.7pt
]

You are a powerful and trustworthy \textbf{Image Editing Evaluator}.

\vspace{0.5em}

Your responsibility is to evaluate whether an image editing model successfully completes a game-based image editing task according to the provided rules.

\vspace{0.5em}

\textbf{Inputs}

\vspace{0.5em}

You will be given the following inputs:

\begin{itemize}
    \item \textbf{Initial Image}: the original game state given to the image editing model.
    \item \textbf{Generated Image}: the output produced by the image editing model.
    \item \textbf{Game Rules}: the rules governing the game and editing constraints.
\end{itemize}

\vspace{0.5em}

\textbf{Evaluation Workflow}

\vspace{0.5em}

\textbf{Step 1: Rule Decomposition}

\vspace{0.4em}

Independently decompose the Game Rules into atomic, visually checkable rules, even if they are already decomposed. The resulting rules must remain consistent with the provided Game Rules.

\vspace{0.5em}

\textbf{Step 2: Rule Violation Checking}

\vspace{0.4em}

Use the Initial Image and the Game Rules to \textbf{infer the expected final state}. \textbf{Compare the Generated Image with this expected state} rule by rule, and determine whether each rule is violated.

\vspace{0.5em}

\textbf{Step 3: Scoring}

\vspace{0.4em}

Let the total number of decomposed atomic rules be $\mathrm{N}$, and let the number of violated rules be $\mathrm{V}$. The score of the Generated Image is:

\[
\textnormal{Score} = \frac{\mathrm{N} - \mathrm{V}}{\mathrm{N}}.
\]

The final score should reflect how consistently the Generated Image follows the Game Rules given the Initial Image.

\vspace{0.5em}

\textbf{Output Requirements}

\vspace{0.5em}

Output only a valid JSON object with exactly two top-level fields: "sub\_questions" and "final\_score". Do not include Markdown code fences or additional text.

\par\vspace{0.7em}
\noindent
\begingroup
\renewcommand{\arraystretch}{1.08}
\begin{tabular}{@{}l@{}}
\{ \\
\hspace*{1.5em}"sub\_questions": [ \\
\hspace*{3em}\{ \\
\hspace*{4.5em}"id": 1, \\
\hspace*{4.5em}"question": "A decomposed rule sub-question", \\
\hspace*{4.5em}"answer": "True or False", \\
\hspace*{4.5em}"explanation": "Brief rule-violation explanation if False; otherwise an empty string" \\
\hspace*{3em}\} \\
\hspace*{1.5em}], \\
\hspace*{1.5em}"final\_score": x \\
\}
\end{tabular}
\endgroup
\par\vspace{0.7em}

Replace the placeholders above with actual evaluation results. Include one entry for each distinct decomposed atomic rule, with consecutive IDs starting at 1.

\vspace{0.5em}

Here, $\mathrm{x}$ is the score calculated in Step 3, expressed as a JSON number rounded to two decimal places.

\vspace{0.5em}

\textbf{Inputs for This Evaluation}

\vspace{0.5em}

\textbf{Game Rules:} \$ \{ game\_rules \}

\vspace{0.5em}

Images are attached in this order: Initial Image, Generated Image. \textless image\textgreater \textless image\textgreater

\end{tcolorbox}

\begin{tcolorbox}[
    title = {Text-to-Image Sub-Question Verifier Prompt},
    breakable,
    colback = orange!5,
    colframe = orange!55!black,
    coltitle = white,
    colbacktitle = orange!55!black,
    fontupper = \small,
    fonttitle = \bfseries\small,
    arc = 1.5mm,
    boxrule = 0.7pt
]

Evaluate the following sub-question about the Generated Image.

\vspace{0.5em}

Please carefully analyze the Generated Image and determine whether it satisfies the requirement expressed in the sub-question. Base your judgment solely on the visual evidence in the Generated Image.

\vspace{0.5em}

If the Generated Image satisfies this requirement, answer "True"; otherwise, answer "False". Evaluate only the requirement expressed in this sub-question.

\vspace{0.5em}

Respond strictly with a JSON object containing exactly these two fields:

\par\vspace{0.7em}
\noindent
\begingroup
\renewcommand{\arraystretch}{1.08}
\begin{tabular}{@{}l@{}}
\{ \\
\hspace*{1.5em}"answer": "True or False", \\
\hspace*{1.5em}"explanation": "Brief error explanation if False; otherwise an empty string" \\
\}
\end{tabular}
\endgroup
\par\vspace{0.7em}

Use the string "True" or "False" for "answer". If the answer is "False", briefly describe the specific mismatch or missing visual evidence in "explanation"; otherwise, use an empty string. Do not include any additional text.

\vspace{0.5em}

\textbf{Inputs for This Evaluation}

\vspace{0.5em}

\textbf{Sub-question:} \$ \{ question \}

\vspace{0.5em}

The attached image is the Generated Image. \textless image\textgreater

\end{tcolorbox}

\begin{tcolorbox}[
    title = {Rule-Governed World Modeling Sub-Question Verifier Prompt-Pairwise},
    breakable,
    colback = orange!5,
    colframe = orange!55!black,
    coltitle = white,
    colbacktitle = orange!55!black,
    fontupper = \small,
    fonttitle = \bfseries\small,
    arc = 1.5mm,
    boxrule = 0.7pt
]

Evaluate the following sub-question about the Generated Image.

\vspace{0.5em}

Please carefully \textbf{compare the Generated Image with the Reference Image} and determine whether the Generated Image satisfies the requirement expressed in the sub-question. Use the Reference Image as the correct standard, with the Initial Image and Game Rules providing the task context.

\vspace{0.5em}

If the Generated Image satisfies this requirement, answer "True"; otherwise, answer "False". Evaluate only the requirement expressed in this sub-question.

\vspace{0.5em}

Respond strictly with a JSON object containing exactly these two fields:

\par\vspace{0.7em}
\noindent
\begingroup
\renewcommand{\arraystretch}{1.08}
\begin{tabular}{@{}l@{}}
\{ \\
\hspace*{1.5em}"answer": "True or False", \\
\hspace*{1.5em}"explanation": "Brief error explanation if False; otherwise an empty string" \\
\}
\end{tabular}
\endgroup
\par\vspace{0.7em}

Use the string "True" or "False" for "answer". If the answer is "False", briefly describe the specific mismatch or missing visual evidence in "explanation"; otherwise, use an empty string. Do not include any additional text.

\vspace{0.5em}

\textbf{Inputs for This Evaluation}

\vspace{0.5em}

\textbf{Sub-question:} \$ \{ question \}

\vspace{0.5em}

\textbf{Game Rules:} \$ \{ game\_rules \}

\vspace{0.5em}

Images are attached in this order: Initial Image, Generated Image, Reference Image. \textless image\textgreater \textless image\textgreater \textless image\textgreater

\end{tcolorbox}

\begin{tcolorbox}[
    title = {Rule-Governed World Modeling Sub-Question Verifier Prompt-Pointwise},
    breakable,
    colback = orange!5,
    colframe = orange!55!black,
    coltitle = white,
    colbacktitle = orange!55!black,
    fontupper = \small,
    fonttitle = \bfseries\small,
    arc = 1.5mm,
    boxrule = 0.7pt
]

Evaluate the following sub-question about the Generated Image.

\vspace{0.5em}

Please analyze the Initial Image, then apply the Game Rules to \textbf{infer the expected final state}. Carefully \textbf{compare the Generated Image with this expected state} and determine whether the Generated Image satisfies the requirement expressed in the sub-question.

\vspace{0.5em}

If the Generated Image satisfies this requirement, answer "True"; otherwise, answer "False". Evaluate only the requirement expressed in this sub-question.

\vspace{0.5em}

Respond strictly with a JSON object containing exactly these two fields:

\par\vspace{0.7em}
\noindent
\begingroup
\renewcommand{\arraystretch}{1.08}
\begin{tabular}{@{}l@{}}
\{ \\
\hspace*{1.5em}"answer": "True or False", \\
\hspace*{1.5em}"explanation": "Brief error explanation if False; otherwise an empty string" \\
\}
\end{tabular}
\endgroup
\par\vspace{0.7em}

Use the string "True" or "False" for "answer". If the answer is "False", briefly describe the specific mismatch or missing visual evidence in "explanation"; otherwise, use an empty string. Do not include any additional text.

\vspace{0.5em}

\textbf{Inputs for This Evaluation}

\vspace{0.5em}

\textbf{Sub-question:} \$ \{ question \}

\vspace{0.5em}

\textbf{Game Rules:} \$ \{ game\_rules \}

\vspace{0.5em}

Images are attached in this order: Initial Image, Generated Image. \textless image\textgreater \textless image\textgreater

\end{tcolorbox}

\end{document}